\documentclass[journal]{IEEEtran}
\usepackage{url}
\usepackage[utf8]{inputenc}
\usepackage[small]{caption}
\usepackage{graphicx}
\usepackage{booktabs}
\usepackage{algorithm}
\usepackage{cite}
\usepackage{algorithmic}
\usepackage{color}
\usepackage{multirow}
\usepackage{bm}
\usepackage{epstopdf}
\usepackage{breakurl}

\usepackage{booktabs}
\usepackage{algorithm,algorithmic}

\usepackage{subfigure}
\usepackage{amsmath,amssymb, amsthm,amsfonts}
\newtheorem{Def}{Definition}

\newtheorem{theorem}{Theorem}
\newtheorem{lemma}{Lemma}
\ifCLASSINFOpdf
\else
\fi
\begin{document}
%
\title{Unbalanced Incomplete Multi-view Clustering\\
 via the Scheme of View Evolution: \\
Weak Views Are Meat; Strong Views Do Eat}

\author{Xiang~Fang,
        Yuchong~Hu,~\IEEEmembership{Member,~IEEE,}
        Pan~Zhou,~\IEEEmembership{Senior Member,~IEEE,}
        and~Dapeng~Oliver~Wu,~\IEEEmembership{Fellow,~IEEE}

\thanks{This work is supported by National Natural Science Foundation of China (NSFC) under grant no. 61972448. (\emph{Corresponding author: Pan~Zhou}.)}
\thanks{X. Fang is with the School of Computer Science and Technology, Key Laboratory of Information Storage System Ministry of Education of China, Huazhong University of Science and Technology, Wuhan 430074, China (e-mail: xfang9508@gmail.com).}
\thanks{Y. Hu is with the School of Computer Science and Technology, Key Laboratory of Information Storage System Ministry of Education of China, Huazhong University of Science and Technology, Wuhan 430074, China (e-mail: yuchonghu@hust.edu.cn).}
\thanks{P. Zhou is with the Hubei Engineering Research Center on Big Data Security, School of Cyber Science and Engineering, Huazhong University of Science and Technology, Wuhan 430074, China (e-mail: panzhou@hust.edu.cn).}
\thanks{D. O. Wu is with the Department of Electrical and Computer Engineering, University of Florida, Gainesville, FL 32611 USA (e-mail: dpwu@ieee.org).}
\thanks{$\copyright$ 2021 IEEE. Personal use of this material is permitted.  Permission from IEEE must be obtained for all other uses, in any current or future media, including reprinting/republishing this material for advertising or promotional purposes, creating new collective works, for resale or redistribution to servers or lists, or reuse of any copyrighted component of this work in other works.}}

\maketitle

\begin{abstract}
Incomplete multi-view clustering is an important technique to deal with real-world incomplete multi-view data.
Previous methods assume that all views have the same incompleteness, i.e., balanced incompleteness.
However, different views often have distinct incompleteness, i.e., unbalanced incompleteness, which results in \emph{strong views} (low-incompleteness views) and \emph{weak views} (high-incompleteness views). The unbalanced incompleteness prevents us from directly using previous methods.
In this paper, inspired by the effective biological evolution theory,
we design the novel scheme of \emph{view evolution} to cluster strong and weak views.
Moreover, we propose an Unbalanced Incomplete Multi-view Clustering method (UIMC), which is the \emph{first} effective method based on view evolution for unbalanced incomplete multi-view clustering.
Compared with previous methods, UIMC has two unique advantages: 1) it proposes weighted multi-view subspace clustering to integrate unbalanced incomplete views, which effectively solves the unbalanced incomplete multi-view clustering problem;
2) it designs the low-rank representation to recover the data, which diminishes the impact of the incompleteness and noises.
Extensive experimental results demonstrate that UIMC improves the clustering performance by up to 40\% on three evaluation metrics over other state-of-the-art methods. We provide codes for all of our experiments in \url{https://github.com/ZeusDavide/TETCI_UIMC}.
\end{abstract}

\begin{IEEEkeywords}
Unbalanced incomplete multi-view clustering, weak view, strong view, view evolution.
\end{IEEEkeywords}

\IEEEpeerreviewmaketitle

\section{Introduction}
\label{section:introduction}
\IEEEPARstart{R}{eal}
data are often with multiple modalities~\cite{8782800,liu2023exploring,wang2025taylor,fang2026towardsicml,kuai2026dynamic,wang2025point,fang2025your,zhang2025monoattack,fang2023hierarchical,liu2024towards,yang2025eood,fang2022multi,fang2026cogniVerse,lei2025exploring,fang2023you,wang2025dypolyseg,fang2025hierarchical,yan2026fit,fang2025adaptive,wang2026topadapter,cai2025imperceptible,fang2026slap,wang2026reasoning,fang2026immuno,wang2026biologically,fang2026disentangling,wang2025reducing,fang2026advancing,fang2026unveiling,wang2026from,liu2023conditional,liu2026attacking,fang2026rethinking,wang2025seeing,fang2026towards,fang2025multi,fang2024fewer,liu2024pandora,fang2024multi,fang2025turing,fang2024not,liu2023hypotheses,fang2024rethinking,liu2024unsupervised,fang2023annotations,xiong2024rethinking,fang2020v,wang2025prototype,zhang2025manipulating,fang2026align,tang2024reparameterization,fang2025adaptivetai,tang2025simplification,fang2021animc,cai2026towards,fang2020double} or collected from different sources~\cite{2016A,WenWZH18,lu2013unified,he2014comment,8114198,gao2018recommendation,gao2019explainable,8387526,8405592}, which are called multi-view data~\cite{8766847,blum1998combining,li2018multi1,Tang2019AAAI,Peng2019,Liu2016}. In many real-world datasets, many views often lose some instances, which results in incomplete views~\cite{zhang2019multi}.
As an illustration, for a multilingual document clustering task, different languages of a document represent distinct views, but we are difficult to
translate each document into all languages~\cite{hu2018doubly,li2017implicit1,wang2019compact}.
The incompleteness will reduce the information available for clustering, which hurts clustering performance.
Due to the ability to integrate these incomplete views, incomplete multi-view clustering has attracted more and more attention~\cite{wang2013collaborative,min2017delicious1}.

Recently, some incomplete multi-view clustering methods have been proposed \cite{zhuge2020joint,tao2019joint,zhuge2019simultaneous}.
%
As a pioneering work, \cite{li2014partial} proposes PVC by learning a common latent subspace for two incomplete views. To couple the instances from two incomplete views, \cite{zhao2016incomplete} proposes IMG by integrating the latent subspace generation and the compact global structure into a unified framework. To handle the situation of multiple incomplete views, \cite{shao2015multiple} proposes MIC by integrating the joint weighted nonnegative matrix factorization~\cite{lee1999learning} and $L_{2,1}$ regularization. To reduce the deviation when the missing rate is large, \cite{hu2018doubly} proposes DAIMC by considering the instance alignment information and aligning different basis matrices simultaneously. To solve the partially view-aligned problem, \cite{huang2020partially} combines representation learning and data alignment. To cluster multi-view
data without parameter selection on cluster size, \cite{peng2019comic} proposes COMIC by automatically learning almost all parameters in a data-driven way.
To learn the common representation for incomplete multi-view clustering, \cite{wen2018incompleteb} proposes IMSC$\_$AGL by exploiting the graph learning~\cite{nie2017self,zong2018weighted,nie2018auto} and spectral clustering techniques. To obtain the robust clustering results, UEAF~\cite{wen2019unified} performs the unified common embedding aligned with incomplete views inferring framework.

\begin{figure}[t]
\centering
\includegraphics[width=0.47\textwidth]{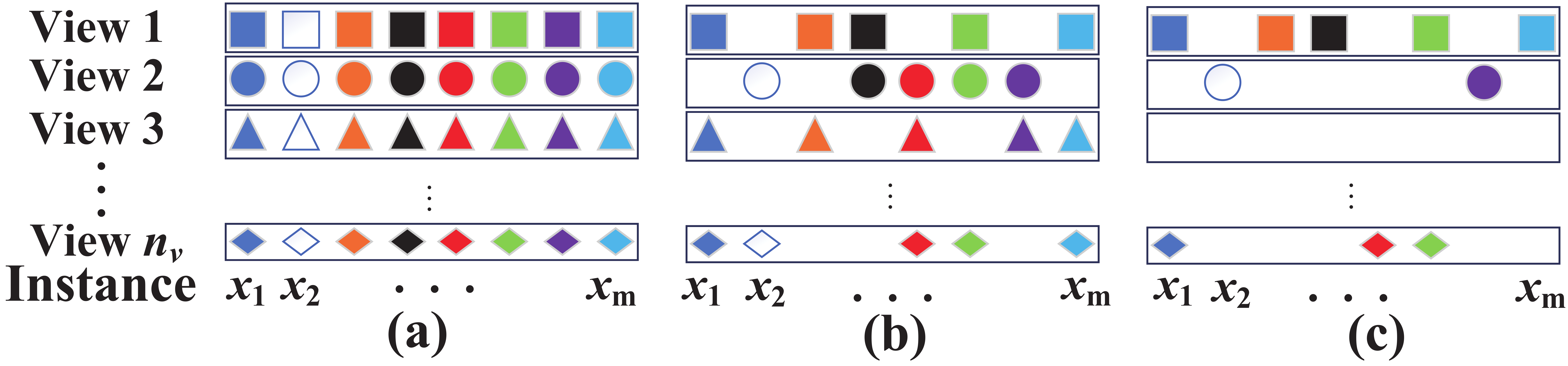}
\caption{Different types of multi-view data, where (a) is complete multi-view data (all views have no missing instances), (b) is balanced incomplete multi-view data (different views have the same number of missing instances), and (c) is unbalanced incomplete multi-view data (different views have the distinct number of missing instances).}
\label{fig:incomplete}
\end{figure}
%
However, almost all incomplete multi-view clustering methods use the view alignment technology for clustering. When using the technology, they always assume that each view has the same incompleteness (called \emph{balanced} incompleteness in Fig.~\ref{fig:incomplete}(b)) and ignore that different views often have distinct incompleteness (called \emph{unbalanced} incompleteness in Fig.~\ref{fig:incomplete}(c)).
For convenience, we name the view with low incompleteness as \emph{strong view} and the view with high incompleteness as \emph{weak view}.
But real-world incomplete multi-view data are often unbalanced incomplete, rather than balanced incomplete.
For example, in web page clustering, different types of data, such as hyperlinks, texts, images, and videos, can be regarded as different views, and most web pages (e.g., Google Scholar and arXiv) only have hyper-links and texts (which are strong views) but not images and videos (which are weak views).
Unfortunately, previous incomplete multi-view clustering methods cannot directly integrate unbalanced incomplete views.
It is because these views have different numbers of presented instances, even some views have no instance (e.g., View 3 in Fig.~\ref{fig:incomplete}(c)). The unbalanced incompleteness makes these methods unable to align different views for clustering.
Therefore, unbalanced incomplete multi-view clustering remains unexplored.

In this paper, we propose an Unbalanced incomplete Multi-view Clustering method (UIMC) to explore the problem.
The biological evolution theory is an effective method to analyze the impact of different traits (strong traits and weak traits) on population development by weighting each trait. Inspired by the theory, our main idea is to integrate strong views and weak views by designing \emph{the scheme of view evolution}, which consists of weighted multi-view subspace clustering (weighted MVSC) and the low-rank representation.
By integrating weighted MVSC and the low-rank representation, our proposed UIMC minimizes the disagreement between each cluster indicator matrix and the consensus, and adaptively learns the proper weights for strong and weak views.
The contributions of UIMC mainly include:
\begin{itemize}
  \item As far as we know, UIMC is a \emph{pioneering} work by introducing the biological evolution theory into clustering algorithms, which has the potential to stimulate new biological ideas and findings in machine learning tracks. Moreover, inspired by biological evolution, UIMC might be the \emph{first} effective method by proposing the scheme of view evolution for unbalanced incomplete multi-view clustering.
  \item By designing weighted MVSC, UIMC directly integrates unbalanced incomplete views by
structuring their Laplacian graphs and weighting these views, which decreases the impact of unbalanced incompleteness.
Also, UIMC minimizes the disagreement between each cluster indicator matrix and the consensus cluster indicator matrix, which aligns the same instances in different views.
Besides, UIMC quantifies the contributions of strong and weak views by
weighting each view properly based on its incompleteness, which makes full use of these presented instances for clustering.
  \item With the help of the low-rank representation, UIMC recovers the data structure for clustering by $\Gamma$-norm and $L_{2,1}$-norm, which further reduces the influence of incompleteness and noises.
  \item Experimental results on four real-world datasets demonstrate its superiority over the state-of-the-art methods. Impressively, based on three evaluation metrics, UIMC raises the clustering performance by more than 40\% in representative cases.
\end{itemize}
The rest of the paper is organized as follows.
Section~\ref{section:back} describes the notations and the related works. Section~\ref{section:meth} first motivates UIMC's main idea, then proposes UIMC, and finally solves it efficiently. Section~\ref{section:exp} evaluates UIMC's performance. Section~\ref{section:con} concludes the paper.

\section{Notations and Related Works}
\label{section:back}
We first define the notations used throughout the paper and then describe three relevant works.

\noindent\textbf{Notation:} Given an original dataset $\{\bm{\overline{X}}^{(v)} \in \mathbb{R}^{d_{v} \times m}, v = 1,\ldots,n_v\}$ with $n_{v}$ views, $m$ instances, and $\bm{\overline{X}}^{(v)}$ is the $v$-th view matrix and $d_{v}$ is the feature dimension of $\bm{\overline{X}}^{(v)}$. We can extract these presented instances from the original view matrix $\bm{\overline{X}}^{(v)}$
and update it to a new view matrix $\bm{X}^{(v)} \in \mathbb{R}^{d_{v} \times k_v}$, which has $k_{v}$ presented instances.
$\bm{M}^{(v)} \in \mathbb{R}^{k_v \times m}$ is an indicator matrix to represent the update.
$\mathbf{1}$ is a column vector with all the elements as one; $\mathbf{0}$ is a zero matrix; $\bm{I}$ is the identity matrix.
For the updated view matrix $\bm{X}^{(v)}$, $w_v$ is its weight, $\bm{E}^{(v)} \in \mathbb{R}^{d_{v} \times k_v}$ is its error matrix, $\bm{A}^{(v)} \in \mathbb{R}^{k_{v} \times k_v}$ is its subspace matrix,
$\bm{F}^{(v)} \in \mathbb{R}^{m \times c}$ is its cluster indicator matrix (where $c$ is the cluster number), $\bm{F}^{\ast} \in \mathbb{R}^{m \times c}$ is its consensus cluster indicator matrix. $||\cdot||_G$ is the $\Gamma$-norm; $||\cdot||_1$ is the $L_1$-norm; $||\cdot||_F$ is the Frobenius norm; $||\cdot||_{2,1}$ is the $L_{2,1}$-norm; $|\cdot|$ denotes the absolute value operation. $\theta$ is a penalty parameter. $\alpha,\beta$ and $\eta$ are nonnegative parameters, each of which is a trade-off between two specific objectives.

\noindent\textbf{Unbalanced incomplete multi-view clustering} is a problem of clustering incomplete multi-view data, where at least two views have different missing rates.
\begin{Def}\label{def_qiangruo}
(Strong view, weak view and dying view). For the $\{v\}_{v=1}^{n_v}$-th view matrix $\bm{\overline{X}}^{(v)}$, $r_v$ denotes the missing rate. $\bm{\overline{X}}^{(v)}$ is the strong view if its missing rate is less than the average  ($r_v<(r_1+r_2+\ldots+r_{n_v})/{n_v}$). $\bm{\overline{X}}^{(v)}$ is the weak view if its missing rate is greater than the average ($r_v>(r_1+r_2+\ldots+r_{n_v})/{n_v}$). $\bm{\overline{X}}^{(v)}$ is the dying view if the number of presented instances is fewer than the number of clusters ($m \times(1-r_v)<c$).
\end{Def}
\subsection{Unbalanced Incomplete Multi-view Data}
For an incomplete multi-view dataset, the $v$-th original view matrix (including missing and presented instances) is represented as $\bm{\overline{X}}^{(v)}\in \mathbb{R}^{d_{v} \times m}$, where $d_{v}$ is the feature dimension and $m$ is the instance number. By removing the columns with missing instances from the view matrix (i.e., preserving these presented instances), we can update the $v$-th original view matrix to a new view matrix $\bm{X}^{(v)}\in \mathbb{R}^{d_v\times k_v}$, where $k_v$ is the number of presented instances ($k_v<m$).
To represent the update, we leverage an indicator matrix $\bm{M}^{(v)}\in \mathbb{R}^{k_v\times m}$ defined as:
\begin{align}\label{M}
  \bm{M}_{i,j}^{(v)} \!=\! \left\{ \begin{array}{ll}
                     1, & \textrm{if the }\bm{X}_i^{(v)}\textrm{ is in }\bm{\overline{X}}_j^{(v)}; \\
                     0, & \textrm{otherwise,}
                   \end{array}
  \right.
\end{align}
\noindent where $\bm{\overline{X}}_j^{(v)}$ is the $j$-th column of $\bm{\overline{X}}^{(v)}$, and
$\bm{X}_i^{(v)}$ is the $i$-th column of $\bm{X}^{(v)}$, which is the $v$-th updated view matrix.
\subsection{Multi-view Subspace Clustering}
As a classic complete multi-view clustering method,
multi-view subspace clustering (MVSC) tries to obtaining a cluster indicator matrix for clustering~\cite{gao2015multi}.
For a complete multi-view dataset $\{\bm{\overline{X}}\}_{v=1}^n$,
MVSC is formulated as:
\begin{align}\label{jm}
J=&\min\limits_{\bm{Z}^{(v)},\bm{E}^{(v)},\bm{F}}\sum_v(||\bm{\overline{X}}^{(v)}-\bm{\overline{X}}^{(v)}\bm{Z}^{(v)}-\bm{E}^{(v)}||_F^2\nonumber\\
&+\beta||\bm{E}^{(v)}||_1+\eta\text{Tr}(\bm{F}^T\bm{L}_Z^{(v)}\bm{F})) \nonumber\\
\mbox{s.t. }&\bm{F}\bm{F}^T=\bm{I},\bm{Z}^{(v)^T}\mathbf{1}=\mathbf{1},\bm{Z}_{i,i}^{(v)}=0,
\end{align}
where for complete view $v$, $\bm{L}_Z^{(v)}$ is its Laplacian graph, $\bm{E}^{(v)}$ is its error matrix,   and $\bm{L}_Z^{(v)}$ is its Laplacian graph, $\bm{Z}^{(v)}\in \mathbb{R}^{m \times m}$ is its subspace matrix, $\bm{Z}_{i,i}^{(v)}=0$ means that each diagonal element of $\bm{Z}$ is 0; $\bm{F}$ is the cluster indicator matrix.
However, MVSC can only cluster complete multi-view data. When dealing with incomplete multi-view data, MVSC will fail due to the inability to learn the cluster indicator matrix from the incomplete view matrix.
\subsection{Biological Evolution Theory}\label{subsection:jinhua}
Biological evolution theory is an effective method to analyze the impact of different biological traits on biological population development.
As a milestone work of the biological evolution theory, ``Survival of the fittest'' is proposed by Herbert Spencer~\cite{spencer1875principles}.
Based on this work,
Darwin proposed Darwin's theory of evolution~\cite{darwin2004origin}, whose key idea is the natural selection from one generation to the next. Given a biological population, the natural selection function ($J_b$) is formulated as
\begin{align}\label{jinhua}
&J_b=\sum_{v=1}^{n_t}a_v\bm{T}^{(v)},\nonumber\\
&\mbox{s.t. }\sum_{v=1}^{n_t}a_v=1,
\end{align}
where $n_t$ is the number of biological traits; $\bm{T}^{(v)}$ is the $v$-th trait matrix; $a_v$ is the ratio of $\bm{T}^{(v)}$.

By changing the ratio of the biological trait, heritable variation drives the biological evolution. Thus, we can rewrite Eq~\eqref{jinhua} as
\begin{align}\label{jiyinjinhua}
&J_g=\sum_{v=1}^{n_t}\overline{a}_v\mu_b\bm{T}^{(v)},\nonumber\\
&\mbox{s.t. }\sum_{v=1}^{n_t}\overline{a}_v=1,
\end{align}
where $\mu_b$ is a parameter to denote heritable variation, $\overline{a}_v\mu=a_v$.

Neutral theory of molecular evolution divides the heritable variation into three types: favorable variation, harmful variation and neutral variation~\cite{kimura1983neutral}. Similarly, we divide $\mu_b$ into three types:
\begin{align}\label{bianyijiyin}
\textrm{The heritable variation is} \left\{ \begin{array}{ll}
                     \textrm{favorable,} & \textrm{if }\mu_b>1; \\
                     \textrm{harmful,}  & \textrm{if }0<\mu_b<1; \\
                     \textrm{neutral,} & \textrm{if }\mu_b=1. \\
                   \end{array}
    \right.
\end{align}
Combining Eq.~\eqref{jiyinjinhua} and Eq.~\eqref{bianyijiyin}, we can see that the favorable variation will increase the ratio of the biological trait; the harmful variation will decrease the ratio; the neutral variation will not change the ratio. Therefore, we can think that the favorable variation has a positive effect; the harmful variation has a negative effect; the neutral update does not affect.

If $\mu_b>1$, the corresponding trait $\bm{T}^{(v)}$ is a strong trait; if $0<\mu_b<1$, $\bm{T}^{(v)}$ is a weak trait.
When the proportion of a trait is much smaller than the average $1/n_t$, the trait (we call it a \emph{dying trait}) often disappears. After  the natural selection, strong traits will have larger ratios, while weak traits will have smaller ratios.
Therefore, we can visually describe the selection as: ``The weak are meat, and the strong do eat\footnote{``Meat'' and ``eating'' do not mean a predator relationship, but a competitive relationship (or strong biological traits eat the weights of the weak biological traits).}''~\cite{mitchell2008cloud}.

After enough generations in a biological population, the proportion of each trait will be approximately unchanged. At this point, we think that the population becomes stable.

\section{Scheme of View Evolution}
\label{section:meth}
We first present the motivation for proposing the scheme of view evolution, then explain the challenges encountered, and finally meet the challenge.

\subsection{Motivation}
By learning a consensus matrix shared by all views (denoted by $\bm{C}$), previous incomplete multi-view clustering methods can
align balanced incomplete views for clustering.
But many real-world datasets often contain unbalanced incomplete views as shown in Fig.~\ref{fig:incomplete}(c) where views 1, 2 and 3 have three, six and eight missing instances, respectively. Thus, these unbalanced incomplete view matrices $\{\bm{X}^{(v)}\}_{v=1}^{n_v} \in \mathbb{R}^{d_{v} \times k_v}$ have distinct numbers of columns and rows, which make it difficult for previous methods to obtain the consensus matrix (i.e., $\bm{C}$) for view alignment.
Thus, different view matrices lose distinct numbers of rows or columns and the unbalanced incompleteness will make it difficult for these methods to obtain available $\bm{C}$.
What is worse, when a view has no instances (i.e., View 3 in Fig.~\ref{fig:incomplete}(c)), these methods will fail during matrix calculation because no instance on the view can be used for calculation.
Also, these methods assume that different views have the same contribution to clustering.
But different views will have unbalanced contributions for clustering because they have different numbers of presented instances.

\begin{table}[t]
\centering
\caption{Connection points between the biological evolution theory and our scheme of view evolution.}
\begin{tabular}{c|c}
\hline
Biological Evolution  & View Evolution \\\hline
trait matrix $\bm{T}^{(v)}$   & view matrix $\bm{X}^{(v)}$ \\\hline
strong trait   & strong view \\\hline
weak trait   & weak view \\\hline
dying trait   & dying view \\\hline
heritable variation   & view variation\\\hline
each generation in the population   & each iteration in the optimization\\\hline
stable biological population   & convergent objective function \\\hline
\end{tabular}
\label{lianjiedian}
\end{table}
The unbalanced contributions will make these methods unable to make full use of the presented instance information, which makes these methods difficult to obtain satisfactory clustering performance.

Biological evolution theory provides a natural method to analyze the changes of strong and weak traits in biological evolution.
Inspired by this, we try to introduce the theory into unbalanced incomplete multi-view clustering.
To facilitate understanding, we give some connection points between the biological evolution theory and our scheme of view evolution in Table~\ref{lianjiedian}:
1) The biological evolution theory uses the trait matrix $\bm{T}^{(v)}$ to represent the $v$-th trait, while our view evolution leverage the view matrix $\bm{X}^{(v)}$ to denote the $v$-th view.
2) Strong traits have larger ratios than weak traits, and strong traits have more information about the  population; strong views have lower incompleteness than weak views, and we can think strong views have more information about the dataset.
3) Dying traits will disappear due to natural selection, while dying views will be removed.
4) The heritable variation can drive the biological evolution, while the view variation can lead to the view evolution.
5) After a generation in the biological population,  the ratio of the biological trait will change; after an iteration in the optimization, the weight of each view will be updated.
6) After enough generations, a biological population will become stable; after several iterations, our objective function will converge.
%
%


In Eq.~\eqref{jinhua}, the biological evolution theory integrates strong and weak traits by assigning a weight to each trait.
Inspired by
the theory, to integrate strong and weak views,
we propose the scheme of view evolution ``weak views are meat; strong views do eat''.
An intuitive idea on the scheme is to directly assign a proper weight to each view (similar to Eq.~\eqref{jinhua}), and then linearly superimpose the objective function of each view. However, when implementing this idea, we have the following challenges:
1) how to bridge unbalanced incomplete views (i.e., construct view representations with the same matrix size for different views); 2) how to group similar instances
on each view;
3) how to quantify the contribution of each view; 4) how to quantify view variation.
In the next section, we will meet these challenges.
\subsection{Weighted MVSC}

\noindent\textbf{1) Bridging unbalanced incomplete views by constructing Laplacian graphs:} When directly using MVSC to process an unbalanced dataset $\{\bm{X}^{(v)}\}_{v=1}^{n_v}\in \mathbb{R}^{d_{v} \times k_v}$,
we will obtain these subspace matrices $\{\bm{A}^{(v)}\}_{v=1}^{n_v}\in \mathbb{R}^{k_v \times k_v}$ with different matrix sizes. Since $k_v$ is different for these views, it is difficult to directly deal with these subspaces for view bridge.
To facilitate view bridge, we construct the following Laplacian graph as the view representation
\begin{align}\label{lapulasitu}
  \bm{M}^{(v)^T}\bm{L}_A^{(v)}\bm{M}^{(v)}=\bm{D}_A^{(v)}-\bm{W}_A^{(v)},
\end{align}
where $\bm{D}_{A_{i,i}}^{(v)}$=$\sum_j\bm{W}_{A_{i,j}}^{(v)}$ and $\bm{W}_A^{(v)}$=$(|\bm{M}^{(v)^T}\bm{A}^{(v)}\bm{M}^{(v)}|+|\bm{M}^{(v)^T}\bm{A}^{(v)^T}\bm{M}^{(v)}|)/{2}$.
Note that the Laplacian graph $\bm{M}^{(v)^T}\bm{L}_A^{(v)}\bm{M}^{(v)} \in \mathbb{R}^{m\times m}$.
Since $m$ is constant,
the Laplacian graphs of different views have the same matrix size, which can be used to bridge unbalanced incomplete views.

\noindent\textbf{2) Aligning views by introducing the disagreement function:} Aligning the same instances in different views is the goal of bridging unbalanced incomplete views. We align different views based on the cluster indicator matrix, which represents the relationship between instances and clusters. Based on the relationship, we can group similar instances into one cluster.
We assume that the $v$-th cluster indicator matrix $\bm{F}^{(v)}$ is a perturbation of the consensus cluster indicator matrix $\bm{F}^{\ast}$. To reduce further the negative influence of the perturbation, we align the same instances in different views by minimizing the disagreement between $\bm{F}^{(v)}$ and $\bm{F}^{\ast}$.
For different datasets, we often need distinct disagreement functions.
We design an adjustable disagreement function $d(\bm{F}^{\ast},\bm{F}^{(v)})$ as follows
\begin{align}\label{F}
d(\bm{F}^{\ast},\bm{F}^{(v)})=||\frac{\bm{F}^{(v)}\bm{F}^{(v)^T}}{||\bm{F}^{(v)}\bm{F}^{(v)^T}||_F^{k}}-\frac{\bm{F}^{\ast}\bm{F}^{\ast^T}}{||\bm{F}^{\ast}\bm{F}^{\ast^T}||_F^{k}}||_F^2,
\end{align}
where $k$ is a positive integer that is adjustable as needed, $\bm{F}^{(v)}\bm{F}^{(v)^T}$ is the similarity matrix of $\bm{F}^{(v)}$, and $\bm{F}^{\ast}\bm{F}^{\ast^T}$ is the similarity matrix of $\bm{F}^{\ast}$. Note that as $k$ changes, we can obtain different disagreement functions. This adjustable function can effectively improve the generalization ability of our scheme.

\noindent\textbf{3) Integrating views by weighting them:}
The alignment in Eq.~\eqref{F} cannot quantify the contribution of each view to clustering. Since strong views have more presented instances and lower incompleteness than weak views, strong views often contain more available data information and have greater contributions.
To make full use of the information for clustering,
we give strong views larger weights than weak views (larger weight means greater contribution).
To judge the contribution of view $v$, we design the following weight $w_v$ based on its incompleteness
\begin{align}\label{w_{v}}
w_{v} \!=\! \left\{ \begin{array}{ll}
                     0, & \textrm{if }k_v<c; \\
                     \frac{\sum_i\sum_j\bm{M}_{i,j}^{(v)}}{\sum_v\sum_i\sum_j\bm{M}_{i,j}^{(v)}}\times n_v, & \textrm{ otherwise,}
                   \end{array}
  \right.
\end{align}
where $\sum_i\sum_j\bm{M}_{i,j}^{(v)}$ denotes the number of presented instances in the $v$-th view, $\sum_v\sum_i\sum_j\bm{M}_{i,j}^{(v)}$ denotes the total number of present instances in all views.
For the case $k_v<c$ (i.e., view $v$ is a dying view) in Eq.~\eqref{w_{v}}, it is incapable to divide $k_v$ instances into $c$ clusters.
To learn available weights for unbalanced incomplete views,
we remove dying views by assigning them zero weights.
When we perform clustering, these views with $w_{v}=0$ will be removed. If the number of dying views is $t_v$, we will update $n_v$ to $n_v-t_v$.
Since strong views have more presented instances than weak views, strong views often contain more data information. To make full use of the presented data information for clustering, we give strong views larger weights than weak views in Eq.~\eqref{w_{v}}.

\noindent\textbf{4) Quantifying view variation:} To drive the view evolution, similar to the biological evolution theory, we need to quantify view variation. Inspired by Section~\ref{subsection:jinhua}, the variation and updating the weight are the same in form.
To update each weight adaptively, we leverage an iterative optimization procedure (see Section~\ref{subsection:opt} in more detail). After each iteration, we can update the weight $w_v$ as follows:
\begin{align}\label{gengw}
w_v=\max(w_{\min},\min(\mu w_v, w_{\max})),
\end{align}
where $w_{\max}$ is the upper bound of $w_v$, $w_{\min}$ is the lower bound of $w_v$. In our paper, we set $w_{\max}=w_v\prod\limits_i(1+0.2^{i})$ and $w_{\min}=w_v\prod\limits_i(1-0.2^{i})$, where $i$ is the iteration number.  $\mu$ is  a parameter to denote heritable variation (i.e., updating $w_v$). To analyze the relative availability of each view, we normalize these weights (i.e., $\sum_vw_v=1$). In our experiments, we use the \emph{normalize} function of MATLAB to normalize them. Based on the effect of $\mu$ on the weight, we divide the update into three types
\begin{align}\label{bianyi}
\textrm{The view variation is} \left\{ \begin{array}{ll}
                     \textrm{favorable,} & \textrm{if }\mu>1; \\
                     \textrm{harmful,}  & \textrm{if }0<\mu<1;\\
                     \textrm{neutral,} & \textrm{if }\mu=1. \\
                   \end{array}
    \right.
\end{align}
Combining Eq.~\eqref{gengw} and Eq.~\eqref{bianyi}, similar to the biological evolution theory, the favorable update will increase the weight and the harmful update will decrease the weight.
To improve the clustering results, we give strong views larger $\mu$ than the weak views at each iteration. Note that we can regard $\mu$ as the update rate of the weight. As the value of $|\mu-1|$ increases, the difference between the weights before and after updating also increase, which will result in faster convergence.
But too fast convergence may lead to failure to learn the optimal weights.
In our experiment, we set $\mu=1.1$ for strong views and $\mu$=$0.9$ for weak views, and our experiment results show the effectiveness of this setting. Therefore, \emph{after each iteration, the weights of strong views will increase, while the weights of weak views will decrease. We call this phenomenon ``strong views eat the weights of weak views''.} Note that we remove dying views rather than weak views in Eq.~\eqref{w_{v}}. Although the weights of weak views will decrease due to Eq.~\eqref{gengw}, we still preserve these weak views for clustering.

Combining Eq.~\eqref{jm}, Eq.~\eqref{F} and Eq.~\eqref{w_{v}}, the final model of our weighted MVSC is as follows:
\begin{align}\label{jwxin}
&\min\limits_{\bm{A}^{(v)},\bm{E}^{(v)},\bm{F^{(v)}},\bm{F}^{\ast}}\sum_v(w_v||\bm{X}^{(v)}-\bm{X}^{(v)}\bm{A}^{(v)}-\bm{E}^{(v)}||_F^2\nonumber\\
&+\beta||\bm{E}^{(v)}||_{1}+\eta\text{Tr}(\bm{F}^{(v)^T}\bm{M}^{(v)^T}\bm{L}_A^{(v)}\bm{M}^{(v)}\bm{F}^{(v)})\nonumber\\
&+\alpha||\frac{\bm{F}^{(v)}\bm{F}^{(v)^T}}{||\bm{F}^{(v)}\bm{F}^{(v)^T}||_F^{k}}-\frac{\bm{F}^{\ast}\bm{F}^{\ast^T}}{||\bm{F}^{\ast}\bm{F}^{\ast^T}||_F^{k}}||_F^2)  \\
&\mbox{s.t. }\bm{F}^{(v)^T}\bm{F}^{(v)}=\bm{I},\bm{F}^{\ast^T}\bm{F}^{\ast}=\bm{I},\bm{A}^{(v)}\mathbf{1}=1,\bm{A}_{i,i}^{(v)}=0,\nonumber
\end{align}
where $\bm{A}^{(v)}\in \mathbb{R}^{k_{v} \times k_v}$.
Note that Eq.~\eqref{jwxin} only integrates different views structurally, and it is difficult to recover the incomplete data because no related technology is used. When a dataset has high incompleteness and several noises, Eq.~\eqref{jwxin} may obtain unsatisfactory clustering performance. In the next section, we will design the low-rank and robust representation model to recover and denoise the incomplete data.


\subsection{Low-rank Representation for Data Recovery}
For a real-world incomplete dataset, the low-rank representation is an effective method to recover the underlying true low-rank data matrix.
The low-rank representation of the subspace is widely used in many subspace-based clustering methods to improve the clustering performance~\cite{zhang2017latent,6589139,7956295,8789678}. Most of these methods assume that the nuclear norm of the coefficient matrix ($||\cdot||_{\ast}$) is balanced to the rank of the coefficient matrix (rank$(\cdot)$). However, the results obtained with this assumption may deviate significantly from the truth~\cite{kang2015robust}.
Therefore, a reasonable low-rank representation of subspace can improve the clustering performance.
$\gamma$-norm~\cite{kang2015robust} is a state-of-the-art nonconvex low-rank representation method, which can get more accurate results than the nuclear norm used by most clustering methods. For any matrix $\bm{B}$, its $\gamma$-norm is defined as:
\begin{align}\label{yuangammafyuan}
||\bm{B}||_{\gamma}=\sum_{i}\frac{(1+\gamma)\epsilon_{B_i}}{\gamma+\epsilon_{B_i}},
\end{align}
where $\gamma$ is a penalty parameter and $\epsilon_{B_i}$ is the $i$-th singular value of $\bm{B}$. We can find that $\lim\limits_{\gamma\rightarrow 0}=$rank($\bm{B}$) and $\lim\limits_{\gamma\rightarrow \infty}=||\bm{B}||_{\ast}$.

However, $\gamma$-norm can only handle single view data. Moreover, we only hope that $\gamma$-norm of $\bm{B}$ tends to rank($\bm{B}$) in most cases.
To obtain a satisfactory low-rank representation, we design the $\Gamma$-norm $||\bm{A}^{(v)}||_G$ of $\bm{A}^{(v)}$ as follows:
%
%
%
%
\begin{align}\label{gfyuan}
||\bm{A}^{(v)}||_G=\sum_{i=1}^{k_{v}}\frac{\epsilon_i^{(v)}}{\epsilon_i^{(v)}+\gamma},
\end{align}
where $\epsilon_i^{(v)}$ is the $i$-th singular value of $\bm{A}^{(v)}$.
Note that if $\gamma\rightarrow 0$, $||\bm{A}^{(v)}||_G\rightarrow$rank$(\bm{A}^{(v)})$, and we choose a small $\gamma$ (e.g., $\gamma=0.001$).
Fig.~\ref{fig:fanshu} compares the rank approximation performance of the nuclear norm and our designed $\Gamma$-norm with different $\gamma$. Obviously, $\Gamma$-norm with $\gamma=0.001$ is closer to the true rank than the nuclear norm.
\begin{figure}[t]
\centering
\includegraphics[width=0.45\textwidth]{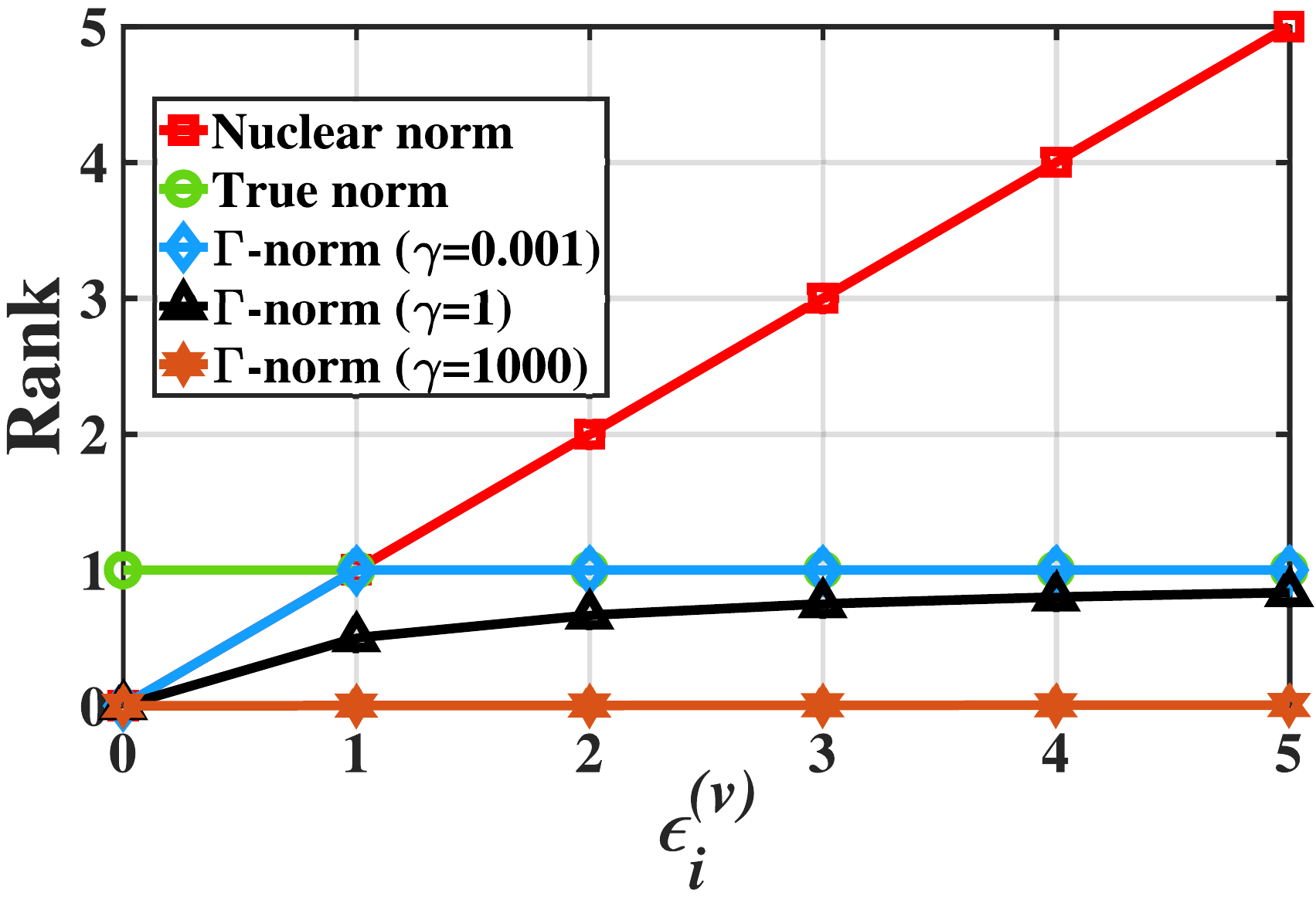}
\caption{The performance of different functions to the rank w.r.t.
a varying positive singular value $\epsilon_i^{(v)}$ (True rank is 1).}
\label{fig:fanshu}
\end{figure}


To learn more satisfactory low-rank representation, we
replace $L_1$-norm in Eq.~\eqref{jwxin} with a relaxed regularization $L_{2,1}$-norm that is robust to noises. $L_{2,1}$-norm of $\bm{E}^{(v)}$ is defined as:
\begin{align}\label{l21_e}
||\bm{E}^{(v)}||_{2,1}=\sum_{i=1}^{d_v}\sqrt{\sum_{j=1}^{k_v}|\bm{E}_{ij}^{(v)}}|^2.
\end{align}
Combining Eq.~\eqref{gfyuan} and Eq.~\eqref{l21_e}, the low-rank representation is formulated as:
\begin{align}\label{l21_xin}
\min\limits_{\bm{A}^{(v)},\bm{E}^{(v)}}||\bm{A}^{(v)}||_G+\beta||\bm{E}^{(v)}||_{2,1}.
\end{align}
Note that $\bm{A}^{(v)}\in \mathbb{R}^{k_{v} \times k_v}$ and $\bm{E}^{(v)} \in \mathbb{R}^{d_{v} \times k_v}$. Therefore, the more instances of a view are presented, the more instance information is extracted by Eq.~\eqref{l21_xin}.
\subsection{Objective Function}
Integrating weighted MVSC (i.e., Eq.~(\ref{jwxin})) and low-rank and robust representation (i.e., Eq.~(\ref{l21_xin})), we can obtain final objective function:
\begin{align}\label{j}
&\min\limits_{\bm{A}^{(v)},\bm{E}^{(v)},\bm{F^{(v)}},\bm{F}^{\ast}}\sum_{v}(||\bm{A}^{(v)}||_G+w_v\theta||\bm{X}^{(v)}-\bm{X}^{(v)}\bm{A}^{(v)}\nonumber\\
-&\bm{E}^{(v)}||_F^2 +\eta\text{Tr}(\bm{F}^{(v)^T}\bm{M}^{(v)^T}\bm{L}_A^{(v)}\bm{M}^{(v)}\bm{F}^{(v)})\\
+&\alpha||\frac{\bm{F}^{(v)}\bm{F}^{(v)^T}}{||\bm{F}^{(v)}\bm{F}^{(v)^T}||_F^{k}}-\frac{\bm{F}^{\ast}\bm{F}^{\ast^T}}{||\bm{F}^{\ast}\bm{F}^{\ast^T}||_F^{k}}||_F^2+\beta||\bm{E}^{(v)}||_{2,1}) \nonumber\\
\mbox{s.t. }&\bm{F}^{(v)^T}\bm{F}^{(v)}=\bm{I},\bm{F}^{\ast^T}\bm{F}^{\ast}=I,\bm{A}^{(v)}\mathbf{1}=1,\bm{A}_{i,i}^{(v)}=0,\nonumber
\end{align}
where $\alpha,\beta$ and $\eta$ are nonnegative parameters, which are used to control the trade-off between two objectives. $\theta$ is a penalty parameter that needs to be tuned in the  optimization (i.e., Section \ref{subsection:opt}).
Note that Eq.~\eqref{j} is a nonconvex function, which is often difficult to optimize. In the next section, we will design an iteration procedure to optimize it.
\subsection{Optimization of Evolution}\label{subsection:opt}
Inspired by Alternating Direction Method of Multipliers (ADMM)~\cite{boyd2011distributed,brbic2018multi},
we introduce two auxiliary variables ($\bm{Q_1}^{(v)}$ and $\bm{Q_2}^{(v)}$) to reformulate Eq.~\eqref{j}
\begin{align}\label{j1}
&\sum\limits_{v}(||\bm{Q_1}^{(v)}||_G+\alpha ||\frac{\bm{F}^{(v)}\bm{F}^{(v)^T}}{||\bm{F}^{(v)}\bm{F}^{(v)^T}||_F^{k}}-\frac{\bm{F}^{\ast}\bm{F}^{\ast^T}}{||\bm{F}^{\ast}\bm{F}^{\ast^T}||_F^{k}}||_F^2\nonumber\\
&+w_v\theta||\bm{X}^{(v)}-\bm{X}^{(v)}\bm{A}^{(v)}-\bm{E}^{(v)}||_F^2+\beta||\bm{E}^{(v)}||_{2,1}\nonumber\\
&+\eta\text{Tr}(\bm{F}^{(v)^T}\bm{M}^{(v)^T}\bm{L}_{Q_2}^{(v)}\bm{M}^{(v)}\bm{F}^{(v)}))\nonumber\\
&\mbox{s.t. }\bm{F}^{(v)}\bm{F}^{(v)^T}=\bm{I},\bm{F}^{\ast}\bm{F}^{\ast^T}=\bm{I},\bm{A}^{(v)}=\bm{Q_1}^{(v)},\nonumber\\
&\bm{A}^{(v)}=\bm{Q_2}^{(v)},{\bm{Q_2}}^{(v)}\mathbf{1}=1,{\bm{Q_2}}_{i,i}^{(v)}=0,
\end{align}
where $\bm{Q_1}^{(v)}$ and $\bm{Q_2}^{(v)}$ are served as the surrogates of $\bm{A}^{(v)}$ to simplify the optimization of Eq.~\eqref{j}.
Therefore, the augmented Lagrangian function of Eq.~\eqref{j1} is formulated as follows:
\begin{align}\label{j2}
&J=\sum\limits_{v}(||\bm{Q_1}^{(v)}||_G+\alpha ||\frac{\bm{F}^{(v)}\bm{F}^{(v)^T}}{||\bm{F}^{(v)}\bm{F}^{(v)^T}||_F^{k}}-\frac{\bm{F}^{\ast}\bm{F}^{\ast^T}}{||\bm{F}^{\ast}\bm{F}^{\ast^T}||_F^{k}}||_F^2\nonumber\\
&+\beta||\bm{E}^{(v)}||_{2,1}+\eta\text{Tr}(\bm{F}^{(v)^T}\bm{M}^{(v)^T}\bm{L}_{Q_2}^{(v)}\bm{M}^{(v)}\bm{F}^{(v)}))+\vartheta^{(v)^T}(1\nonumber\\
&-{\bm{Q_2}}^{(v)}\mathbf{1})+\frac{w_v\theta}{2}(||\bm{X}^{(v)}-\bm{X}^{(v)}\bm{A}^{(v)}-\bm{E}^{(v)}+\frac{\bm{H_0}^{(v)}}{w_v\theta}||_F^2\nonumber\\
&+||\bm{A}^{(v)}-\bm{Q_1}^{(v)}+\frac{\bm{H_1}^{(v)}}{w_v\theta}||_F^2+||\bm{A}^{(v)}-\bm{Q_2}^{(v)}+\frac{\bm{H_2}^{(v)}}{w_v\theta}||_F^2),
\end{align}
where $\bm{H_0}^{(v)}$, $\bm{H_1}^{(v)}$, $\bm{H_2}^{(v)}$ and $\vartheta^{(v)}$ are Lagrange multipliers.
Note that Eq.~\eqref{j2} is not convex for all variables simultaneously, and it is difficult to solve Eq.~\eqref{j2} in one step. Thus, we design the following nine-step procedure to update each variable iteratively~\cite{kang2015robust}.

\noindent\textbf{Step 1.} Updating $\bm{F}^{(v)}$.
Fixing the other variables, for the update of $\bm{F}^{(v)}$, we need to minimize
\begin{align}\label{jfyuan}
J(\bm{F}^{(v)})&= ||\frac{\bm{F}^{(v)}\bm{F}^{(v)^T}}{||\bm{F}^{(v)}\bm{F}^{(v)^T}||_F^{k}}-\frac{\bm{F}^{\ast}\bm{F}^{\ast^T}}{||\bm{F}^{\ast}\bm{F}^{\ast^T}||_F^{k}}||_F^2\nonumber\\
&+\eta\text{Tr}(\bm{F}^{(v)^T}\bm{M}^{(v)^T}\bm{L}_{Q_2}^{(v)}\bm{M}^{(v)}\bm{F}^{(v)}).
\end{align}
Since $||\bm{F}^{(v)}\bm{F}^{(v)^T}||_F^2=||\bm{F}^{\ast}\bm{F}^{\ast^T}||_F^2=c$, we can rewrite Eq.~\eqref{jfyuan} as
\begin{align}\label{jfzhongjian}
J(\bm{F}^{(v)})&=\frac{2\alpha( c-\text{Tr}(\bm{F}^{\ast}\bm{F}^{\ast^T}\bm{F}^{(v)}\bm{F}^{(v)^T}))}{c^{k}}\nonumber\\
&+\eta\text{Tr}(\bm{F}^{(v)^T}\bm{M}^{(v)^T}\bm{L}_{Q_2}^{(v)}\bm{M}^{(v)}\bm{F}^{(v)}),
\end{align}
Since $c$ and $k$ are constant, we can ignore them for simplicity. Thus, we also rewrite Eq.~\eqref{jfzhongjian} as
\begin{align}\label{jf}
&J(\bm{F}^{(v)})=\eta\text{Tr}(\bm{F}^{(v)^T}\bm{Y}^{(v)}\bm{F}^{(v)})-\alpha\text{Tr}(\bm{O}\bm{F}^{(v)}\bm{F}^{(v)^T})  \nonumber\\
\Longrightarrow &J(\bm{F}^{(v)})=\text{Tr}(\bm{F}^{(v)^T}(\eta \bm{Y}^{(v)}-\alpha\bm{O})\bm{F}^{(v)}),
\end{align}
where $\bm{Y}^{(v)}=\bm{M}^{(v)^T}\bm{L}_{Q_2}^{(v)}\bm{M}^{(v)}$ and $\bm{O}=\bm{F}^{\ast}\bm{F}^{\ast^T}$. Eq.~\eqref{jf} is a standard spectral clustering object with graph Laplacian matrix ($\eta\bm{Y}^{(v)}-\alpha \bm{O}^{(v)}$). By calculating the $c$ eigenvectors of Laplacian matrix corresponding to the $c$ smallest eigenvalues, we can obtain the optimal solution $\bm{F}^{(v)}$.

\noindent\textbf{Step 2.} Updating $\bm{A}^{(v)}$.
%
Similar to $\bm{F}^{(v)}$, we have
\begin{align}\label{jzpian}
\bm{A}^{(v)}&=(\bm{X}^{(v)^T}\bm{X}^{(v)}+2\bm{I})^{-1}(\bm{X}^{(v)^T}(\bm{X}^{(v)}+\frac{\bm{H_0}^{(v)}}{w_v\theta}-\bm{E}^{(v)}) \nonumber\\
&+\bm{Q_1}^{(v)}+\bm{Q_2}^{(v)}-\frac{\bm{H_1}^{(v)}}{w_v\theta}-\frac{\bm{H_2}^{(v)}}{w_v\theta}).
\end{align}
Since $(\bm{X}^{(v)^T}\bm{X}^{(v)}+2\bm{I})^{-1}$ is constant, we can pre-compute it before our optimization to reduce the computational cost.

\noindent\textbf{Step 3.} Updating $\bm{Q_1}^{(v)}$.
Fixing the other variables, we update $\bm{Q_1}^{(v)}$ by solving
\begin{align}\label{jq1}
\bm{Q_1}^{(v)}=\arg\min\limits_{\bm{Q_1}^{(v)}}||\bm{Q_1}^{(v)}||_G+\frac{w_v\theta}{2}||\bm{Q_1}^{(v)}-\bm{T}^{(v)}||_F^2,
\end{align}
where $\bm{T}^{(v)}$=$\bm{A}^{(v)}+\bm{H_1}^{(v)}/({w_v\theta})$.
Note that $||\bm{Q_1}^{(v)}||_G$ is the nonconvex surrogate of rank$(\bm{Q_1}^{(v)})$, and the subproblem of updating $\bm{Q_1}^{(v)}$ is a nonconvex function. Thus, it is difficult for us to directly obtain a closed-form solution to this problem. Fortunately, based on the Moreau-Yosida regularization technology and difference of convex (DC) programming \cite{tao1997convex}, we can decompose Eq.~\eqref{jq1} as the difference of two convex functions to obtain the closed-form solution as follows.

At the $(t+1)$ th iteration ($t>0$), we update $\bm{Q_1}^{(v)^{t+1}}$ by solving the
following subproblem:
\begin{align}\label{jq1}
\bm{Q_1}^{(v)^{t+1}}=\arg\min\limits_{\bm{Q_1}^{(v)^t}}||\bm{Q_1}^{(v)^t}||_G+\frac{w_v^t\theta}{2}||\bm{Q_1}^{(v)^t}-\bm{T}^{(v)^t}||_F^2.
\end{align}
To solve Eq.~\eqref{jq1}, we first develop the following theorem and the corresponding proof.
\begin{theorem}\label{mfanshudingli}
Let $\bm{T}=\bm{U}\bm{\Sigma}_T\bm{V}^T$ be the SVD of $\bm{T}$, where $\bm{\Sigma}_T=diag(\sigma_T)$. Set $H(\bm{Q_1}^{(v)})=||\bm{Q_1}^{(v)^t}||_G=h\circ\sigma_Q$
be a unitarily invariant function.
\begin{align}\label{hm}
\min\limits_{\bm{Q_1}^{(v)}}H(\bm{Q_1}^{(v)})+\frac{\omega}{2}||\bm{Q_1}^{(v)}-\bm{A}||_F^2.
\end{align}
Then an optimal solution to the following problem is $\bm{Q}^{\ast}=\bm{U}\bm{\Sigma}_Q^{\ast}\bm{V}^T$, where $\bm{\Sigma}_Q^{\ast}=diag(\sigma^{\ast})$ and $\sigma^{\ast}=\text{prox}_{h,\omega}(\sigma_T)$. $\text{prox}_{h,\omega}(\sigma_T)$ is the \emph{Moreau-Yosida operator}, defined as
\begin{align}\label{prox}
\text{prox}_{h,\omega}(\sigma_T)=\arg\min\limits_{\sigma}h(\sigma)+\frac{\omega}{2}||\sigma-\sigma_T||_2^2.
\end{align}
\end{theorem}
\begin{proof}\label{mfanshuzhengming}
Since $\bm{T}=\bm{U}\bm{\Sigma}_T\bm{V}^T$, then $\bm{\Sigma}_T=\bm{U}^T\bm{T}\bm{V}$. Denoting $\bm{M}^{(v)}=\bm{U}^{(v)^T}\bm{Q_1}^{(v)}\bm{V}^{(v)}$ which has exactly the same singular values as $\bm{Q_1}^{(v)}$, we have
\begin{subequations}
\begin{align}
&H(\bm{Q_1}^{(v)})+\frac{\omega}{2}||\bm{Q_1}^{(v)}-\bm{T}||_F^2\label{Z0}\\
=&H(\bm{M}^{(v)})+\frac{\omega}{2}||\bm{M}^{(v)}-\bm{\Sigma}_T||_F^2,\label{ZA}\\
\geq&H(\bm{\Sigma}_M^{(v)})+\frac{\omega}{2}||\bm{\Sigma}_M^{(v)}-\bm{\Sigma}_T||_F^2,\label{ZB}\\
=&H(\bm{\Sigma}_{Q_1}^{(v)})+\frac{\omega}{2}||\bm{\Sigma}_{Q_1}^{(v)}-\bm{\Sigma}_T||_F^2 ,\label{ZC}\\
=&h(\sigma)+\frac{\omega}{2}||\sigma-\sigma_T||_2^2,\label{ZD}\\
\geq&h(\sigma^{\ast})+\frac{\omega}{2}||\sigma^{\ast}-\sigma_T||_2^2.\nonumber
\end{align}
\end{subequations}
Note that Eq.~\eqref{ZA} holds since the Frobenius norm is unitarily invariant; Eq.~\eqref{ZB} is due to the Hoffman-Wielandt inequality; and Eq.~\eqref{ZC} holds as $\bm{\Sigma}_{Q_1}^{(v)}=\bm{\Sigma}_M$. Thus, Eq.~\eqref{ZC} is a lower bound
of Eq.~\eqref{Z0}. Because $\bm{\Sigma}_M^{(v)}=\bm{\Sigma}_{Q_1}^{(v)}=\bm{Q_1}^{(v)}=\bm{U}^{(v)^T}\bm{M}^{(v)}\bm{V}^{(v)}$, the SVD of $\bm{M}^{(v)}$ is $\bm{M}^{(v)}=\bm{U}^{(v)^T}\bm{\Sigma}_M^{(v)}\bm{V}^{(v)}$. By minimizing Eq.~\eqref{ZD}, we get $\sigma^{\ast}$. Hence $\bm{M}^{\ast}=\bm{U}diag(\sigma^{\ast})\bm{V}^T$, which is the optimal solution of Eq.~\eqref{hm}. Thus, Theorem~\ref{mfanshudingli} is proved.
\end{proof}
Note that $||\bm{Q_1}^{(v)}||_G$ is the nonconvex surrogate of rank$(\bm{Q_1}^{(v)})$, and the subproblem of updating $\bm{Q_1}^{(v)}$ is a nonconvex function. Thus, it is difficult for us to directly obtain a closed-form solution to this problem. Fortunately, based on the Moreau-Yosida regularization technology and difference of convex (DC) programming \cite{tao1997convex}, we can decompose Eq.~\eqref{jq1} as the difference of two convex functions and iteratively optimizes it by linearizing the concave term at each iteration. For the $(t+1)$ th iteration, we have
\begin{align}\label{ji}
\sigma^{t+1}=\arg\min h(\sigma^t)+\frac{\omega^t}{2}||\sigma^t-\sigma_T^t||_2^2,
\end{align}
which admits a closed-form solution~\cite{6574874}
\begin{align}\label{bishijie}
\sigma^{t+1}=(\sigma_T^t-\frac{\varphi_t}{\omega^t})_+,
\end{align}
where $\varphi_t=\partial h(\sigma^t)$ is the gradient of $h(\cdot)$ at $\sigma^t$, and $\bm{U}^{(v)}diag(\sigma_T^t)^{(v)}\bm{V}^{(v)^T}$ is the SVD of ($\bm{A}^{(v)}+\bm{H_1}^{(v)}/({w_v\theta})$). After several iterations, it converges to a locally optimal point $\sigma^{\ast}$. Then $\bm{Q_1}^{(v)^{t+1}}=\bm{U}^{(v)}diag(\sigma^{(v)^{\ast}})\bm{V}^{(v)^T}$.

\noindent\textbf{Step 4.} Updating $\bm{Q_2}^{(v)}$.
Similar to $\bm{F}^{(v)}$ and \cite{gao2015multi}, we can update $\bm{Q_2}^{(v)}$ as follows:
\begin{align}\label{q2jie}
{\bm{Q_2}}_{i,j}^{(v)}=\! \left\{ \begin{array}{ll}
                     \bm{R}_{i,:}^{(v)}-\frac{\eta}{2w_v\theta}\bm{S}_{i,:}^{(v)}, & \textrm{if }\bm{R}_{i,:}^{(v)}\geq\frac{\eta}{2w_v\theta}\bm{S}_{i,:}^{(v)}; \\
                     \bm{R}_{i,:}^{(v)}+\frac{\eta}{2w_v\theta}\bm{S}_{i,:}^{(v)},  & \textrm{if }\bm{R}_{i,:}^{(v)}<\frac{\eta}{2w_v\theta}\bm{S}_{i,:}^{(v)},
                   \end{array}
  \right.
\end{align}
where $\bm{S}_{i,j}^{(v)}$=$||\bm{G}_{i,:}^{(v)}-\bm{G}_{j,:}^{(v)}||_2^2$, $\bm{G}_{i,:}^{(v)}$ and $\bm{G}_{j,:}^{(v)}$ is the $i$-th and $j$-th row vectors of $\bm{G}^{(v)}$, respectively. $\bm{R}_{i,:}^{(v)}$=$\bm{A}_{i,:}^{(v)}+{\bm{H_2}}_{i,:}^{(v)}/({w_v\theta})$ and $\bm{G}^{(v)}$=$\bm{M}^{(v)}\bm{F}^{(v)}$.

\noindent\textbf{Step 5.} Updating $\vartheta^{(v)}$.
Due to ${\bm{Q_2}}_{i,:}^{(v)}\mathbf{1}$=$1$ with Eq.~\eqref{q2jie}, the Lagrange multiplier can be updated as follows:
\begin{align}\label{q2chengzi}
\vartheta_i^{(v)}=\frac{w_v\theta(\sum_{i,j}({\bm{Q_2}}_{i,j}^{(v)}-\frac{{\bm{H_2}}_{i,:}^{(v)}}{w_v\theta})-1)}{1-n_v}.
\end{align}
\noindent\textbf{Step 6.} Updating $\bm{E}^{(v)}$.
%
Similar to $\bm{F}^{(v)}$, we update $\bm{E}^{(v)}$ as follows:
\begin{align}\label{ejie}
\bm{E}^{(v)}=(\frac{\beta}{w_v\theta}\bm{D}^{(v)}+\bm{I})^{-1}\bm{K}^{(v)},
\end{align}
where $\bm{D}^{(v)}$=$\partial||\bm{E}^{(v)}||_{2,1}/{\partial \bm{E}^{(v)}}\bm{E}^{(v)^{-1}}$, $\bm{K}^{(v)}$=$\bm{X}^{(v)}-\bm{X}^{(v)}\bm{A}^{(v)}+\bm{H_0}^{(v)}/({w_v\theta})$.
%

\noindent\textbf{Step 7.} Updating $\bm{F}^{\ast}$.
Similar to $\bm{F}^{(v)}$, we minimize the following objective function:
\begin{align}\label{jfxing}
&-\alpha\text{Tr}(\bm{F}^{\ast}\bm{F}^{\ast^T}\bm{F}^{(v)}\bm{F}^{(v)^T})
\nonumber\\
\Longrightarrow &\text{Tr}(\bm{F}^{\ast^T}(-\bm{F}^{(v)}\bm{F}^{(v)^T})\bm{F}^{\ast}).
\end{align}
We can obtain the optimal solution $\bm{F}^{(v)}$ by calculating the $c$ eigenvectors of $\sum_v(\bm{F}^{(v)}\bm{F}^{(v)^T})$ corresponding to the $c$ largest eigenvalues.

\noindent\textbf{Step 8.} Updating $\bm{H_0}^{(v)},\bm{H_1}^{(v)}$, $\bm{H_2}^{(v)}$. Similar to $\bm{F}^{(v)}$, we update $\bm{H_0}^{(v)},\bm{H_1}^{(v)}$ and $\bm{H_2}^{(v)}$ by
\begin{align}\label{gengh}
\bm{H_0}^{(v)}&=\bm{H_0}^{(v)}+w_v\theta(\bm{X}^{(v)}-\bm{X}^{(v)}\bm{A}^{(v)}-\bm{E}^{(v)})\nonumber\\
\bm{H_1}^{(v)}&=\bm{H_1}^{(v)}+w_v\theta(\bm{A}^{(v)}-\bm{Q_1}^{(v)})\\
\bm{H_2}^{(v)}&=\bm{H_2}^{(v)}+w_v\theta(\bm{A}^{(v)}-\bm{Q_2}^{(v)})\nonumber
\end{align}
\noindent\textbf{Step 9.} Updating $w_v$ and $\theta$.
Fixing the other variables, we update $w_v$ by Eq.~\eqref{gengw}. Similarly, we update $\theta$ by
\begin{align}\label{gengseita}
\theta=\min(\phi\theta,{\theta}_{\max}),
\end{align}
where $\phi$ and ${\theta}_{\max}$ are constant. We set $\phi=1.1$ and ${\theta}_{\max}=10^7$ in our experiments.
\subsection{Convergence Analysis}
Note that the optimization of UIMC is simplified as nine steps. Each step can obtain the closed solution w.r.t. the corresponding variable. The objective function is bounded and all the above nine steps do not increase the objective function value.
For Step 1, Step 2, and Step 4-9, we can easily find their objective functions are bounded and these steps do not increase the objective function value. Therefore, we try to analyze the impact of Step 3 on the objective function in our submission.

For convenience, we write $||\bm{Q_1}^{(v)}||_{\Gamma}$ as $H(\bm{Q_1}^{(v)})$ in Eq. (16) in our submission.
\begin{align}\label{gaimubiaohanshu}
&J(\bm{Q_1}^{(v)},\bm{A}^{(v)},\bm{H_1}^{(v)},w_v)=\sum\limits_{v}(H(\bm{Q_1}^{(v)})\\
&+\frac{w_v\theta}{2}||\bm{A}^{(v)}-\bm{Q_1}^{(v)}||_F^2+<\frac{\bm{H_1}^{(v)}}{w_v\theta},\bm{A}^{(v)}-\bm{Q_1}^{(v)}>),\nonumber
\end{align}
where $<\cdot,\cdot>$ is the inner product of two matrices~\cite{8543490}.
\begin{lemma}\label{bianhuajia}
$\bm{Q_1}^{(v)^t}$ is bounded.
\end{lemma}
\begin{proof}\label{bianhuazhengming}
With some algebra, we can obtain
\begin{align}\label{denghaobian}
  &J(\bm{Q_1}^{(v)^t},\bm{A}^{(v)^t},\bm{H_1}^{(v)^t},w_v^t)\nonumber\\
  =&J(\bm{Q_1}^{(v)^{t}},\bm{A}^{(v)^{t}},\bm{H_1}^{(v)^{t-1}},w_v^{t-1})\nonumber\\
  &+\frac{(w_v^t-w_v^{t-1})\theta}{2}||\bm{A}^{(v)}-\bm{Q_1}^{(v)}||_F^2\nonumber\\
  &+\text{Tr}[(\bm{H_1}^{(v)^t}-\bm{H_1}^{(v)^{t-1}})(\bm{A}^{(v)}-\bm{Q_1}^{(v)})]\nonumber\\
  =&J(\bm{Q_1}^{(v)^{t}},\bm{A}^{(v)^{t}},\bm{H_1}^{(v)^{t-1}},w_v^{t-1})\nonumber\\
  &+\frac{w_v^t+w_v^{t-1}}{2\theta(w_v^{t-1})^2}||\bm{H_1}^{(v)^t}-\bm{H_1}^{(v)^{t-1}}||_F^2.
\end{align}
Then,
\begin{align}\label{budengshilian}
&J(\bm{Q_1}^{(v)^{t+1}},\bm{A}^{(v)^{t+1}},\bm{H_1}^{(v)^t},w_v^t)\nonumber\\
  \leq&J(\bm{Q_1}^{(v)^{t+1}},\bm{A}^{(v)^t},\bm{H_1}^{(v)^t},w_v^t)\nonumber\\
   \leq&J(\bm{Q_1}^{(v)^t},\bm{A}^{(v)^t},\bm{H_1}^{(v)^t},w_v^t)\nonumber\\
  \leq&J(\bm{Q_1}^{(v)^t},\bm{A}^{(v)^t},\bm{H_1}^{(v)^{t-1}},w_v^{t-1})\nonumber\\
  &+\frac{w_v^t+w_v^{t-1}}{2\theta(w_v^{t-1})^2}||\bm{H_1}^{(v)^t}-\bm{H_1}^{(v)^{t-1}}||_F^2.
\end{align}
Iterating the inequality chain~\eqref{budengshilian} $t$ times, we obtain
\begin{align}\label{bianhuanhuode}
  &J(\bm{Q_1}^{(v)^{t+1}},\bm{A}^{(v)^{t+1}},\bm{H_1}^{(v)^t},w_v^t)\nonumber\\
  \leq&J(\bm{Q_1}^{(v)^1},\bm{A}^{(v)^1},\bm{H_1}^{(v)^0},w_v^0)\nonumber\\
  &+\sum_{i=1}^t\frac{w_v^i+w_v^{i-1}}{2\theta(w_v^{i-1})^2}||\bm{H_1}^{(v)^t}-\bm{H_1}^{(v)^{t-1}}||_F^2
\end{align}
Based on the definition of $w_v$ and $\theta$ in our submission, we have $\sum_{i=1}^{\infty}(w_v^i+w_v^{i-1})/(2\theta(w_v^{i-1})^2)<\infty$.
Since $||\bm{H_1}^{(v)^i}-\bm{H_1}^{(v)^{i-1}}||_F^2$ is bounded, all terms on the righthand side of the above inequality are bounded, thus $J(\bm{Q_1}^{(v)^{t+1}},\bm{A}^{(v)^{t+1}},\bm{H_1}^{(v)^t},w_v^t)$ is upper bounded.

Besides,
\begin{align}\label{genggaihuode}
  &J(\bm{Q_1}^{(v)^{t+1}},\bm{A}^{(v)^{t+1}},\bm{H_1}^{(v)^t},w_v^t)+\frac{1}{2\theta w_v^t}||\bm{H_1}^{(v)^{t}}||_F^2\nonumber\\
  =&H(\bm{Q_1}^{(v)^{t+1}})+\frac{w_v\theta}{2}||\bm{A}^{(v)^{t+1}}-\bm{Q_1}^{(v)^{t+1}}+\frac{\bm{H_1}^{(v)^t}}{w_v\theta}||_F^2.
  \end{align}
Since each term on the right-hand side is bounded, $\bm{A}^{(v)^{t+1}}$ is bounded. By the last term on the right-hand of Eq.~\eqref{genggaihuode}, $\{\bm{Q_1}^{(v)^t}\}$ is bounded. Therefore, both $\{\bm{Q_1}^{(v)^t}\}$ and $\{\bm{A}^{(v)^t}\}$ are bounded.
\end{proof}
\begin{lemma}\label{bianhua2}
Let $\{\bm{Q_1}^{(v)^t},\bm{A}^{(v)^t},\bm{H_1}^{(v)^t}\}$ be the sequence and $\{\bm{Q_1}^{(v)^{\ast}},\bm{A}^{(v)^{\ast}},\bm{H_1}^{(v)^{\ast}}\}$ be an accumulation point.
Then $\{\bm{Q_1}^{(v)^{\ast}},\bm{A}^{(v)^{\ast}}\}$ is a stationary point if $\lim\limits_{t\rightarrow\infty}w_v^t(\bm{A}^{(v)^{t+1}}-\bm{A}^{(v)^{t}})\rightarrow 0$.
\end{lemma}
\begin{proof}
The sequence $\{\bm{Q_1}^{(v)^t},\bm{Z}^{(v)^t},\bm{H_1}^{(v)^t}\}$ is bounded as shown in Lemma~\ref{bianhua2}. By Bolzano-Weierstrass theorem, the sequence must have at
least one accumulation point, e.g., $\{\bm{Q_1}^{(v)^{\ast}},\bm{A}^{(v)^{\ast}},\bm{H_1}^{(v)^{\ast}}\}$. Without loss of generality, we assume that $\{\bm{Q_1}^{(v)^t},\bm{A}^{(v)^t},\bm{H_1}^{(v)^t}\}$ itself converges to $\{\bm{Q_1}^{(v)^{\ast}},\bm{A}^{(v)^{\ast}},\bm{H_1}^{(v)^{\ast}}\}$.

Since $\bm{A}^{(v)^t}-\bm{Q_1}^{(v)^t}=(\bm{H_1}^{(v)^t}-\bm{H_1}^{(v)^{t-1}})/(w_v^t\theta)$, we have $\lim\limits_{t\rightarrow\infty}\bm{A}^{(v)^t}-\bm{Q_1}^{(v)^t}=0$. Thus the primal feasibility condition is satisfied.

For $\bm{Q_1}^{(v)^{t+1}}$, it is true that
\begin{align}\label{lapulasitu}
 &\partial_{\bm{Q_1}}(\bm{Q_1}^{(v)^{t+1}},\bm{A}^{(v)^t},\bm{H_1}^{(v)^t},w_v^t)|_{\bm{Q_1}^{(v)^{t+1}}}\nonumber\\
 =&\partial_{\bm{Q_1}}H(\bm{Q_1}^{(v)^{t+1}})+\bm{H_1}^{(v)^t}+w_v^t\theta(\bm{A}^{(v)^t}-\bm{Q_1}^{(v)^t})\\
  =&\partial_{\bm{Q_1}}H(\bm{Q_1}^{(v)^{t+1}})+\bm{H_1}^{(v)^{t+1}}+w_v^t\theta(\bm{A}^{(v)^{t+1}}-\bm{A}^{(v)^t})=0.\nonumber
 \end{align}
 If the singular value decomposition of $\bm{Q_1}^{(v)}$ is $\bm{U}^{(v)}diag(\sigma_i^{(v)})\bm{V}^{(v)^T}$ , according to Theorem~\ref{mfanshudingli},
 \begin{align}\label{lapulasitu}
\partial_{\bm{Q_1}}H(\bm{Q_1}^{(v)^{t+1}})|_{\bm{Q_1}^{(v)^{t+1}}}=\bm{U}diag(\tau^{(v)})\bm{V}^{(v)^T},
 \end{align}
 where $\tau_i=\gamma/(\gamma+\sigma_i)^2$ if $\sigma_i\neq 0$; otherwise, it is $1/\gamma$.
Since $\sigma_i\in(0,1/\gamma]$ is finite, $H(\bm{Q_1}^{(v)^{t+1}})|_{\bm{Q_1}^{(v)^{t+1}}}$ is bounded.
Since $\bm{A}^{(v)^t}$ is bounded, $w_v^t\theta(\bm{A}^{(v)^{t+1}}-\bm{A}^{(v)^t})$ is bounded.
Under the assumption that $w_v^t\theta(\bm{A}^{(v)^{t+1}}-\bm{A}^{(v)^t})\rightarrow 0$,
\begin{align}\label{lapulasitu}
\partial_{\bm{Q_1}}H(\bm{Q_1}^{(v)^{t+1}})|_{\bm{Q_1}^{(v)^{t+1}}}+\bm{H_1}^{(v)^{\ast}}=0.
 \end{align}
 Hence, $\{\bm{Q_1}^{(v)^{\ast}},\bm{A}^{(v)^{\ast}},\bm{H_1}^{(v)^{\ast}}\}$ satisfies the KKT conditions of $J(\bm{Q_1}^{(v)^{t+1}},\bm{Z}^{(v)^{t+1}},\bm{H_1}^{(v)^t})$. Thus $\{\bm{Q_1}^{(v)^{\ast}},\bm{A}^{(v)^{\ast}}\}$ is a stationary point of the objective function.
\end{proof}

Thus, the objective function can reduce monotonically to a stationary value and UIMC can at least find a locally optimal solution. When we initialize each variable to obtain this solution, different initialization methods may affect the convergence speed, but UIMC will eventually reach convergence due to the monotonic decline of our objective function.
\begin{table}[t]
\centering
\caption{Statistics of datasets.}
\begin{tabular}{c|cccc}
\hline
Dataset  & BUAA & 3-Sources  & BBC & Digit \\\hline
\# views   & 2& 3  & 4& 5 \\\hline
\# instances   & 180& 169  & 685& 2000 \\\hline
\# clusters   & 20 & 6 & 5& 10  \\\hline
\end{tabular}
\label{dataset}
\end{table}
\begin{table}[t]
\centering
\caption{The missing rate vector of each view.}
 {\scalebox{1.0}{\begin{tabular}{cccc}
\hline
BUAA & 3-Sources  & BBC & Digit \\\hline
VIS (1$\times\bm{r}$) & BBC (0.2$\times\bm{r}$) & Ba (0.25$\times\bm{r}$)& FAC (0.25$\times\bm{r}$) \\
NIR (1$\times\bm{r}$)  & REU (1.0$\times\bm{r}$)& Bb (0.75$\times\bm{r}$)&FOU (0.75$\times\bm{r}$)     \\
- & GUA (1.8$\times\bm{r}$)&Bc (1.25$\times\bm{r}$) & KAR (1.00$\times\bm{r}$)   \\
-  & -  & Bd (1.75$\times\bm{r}$)& PIX (1.25$\times\bm{r}$)  \\
 -  & - &-& ZER (1.75$\times\bm{r}$)    \\\hline
\end{tabular}}}
\label{datashezhi}
\end{table}
\begin{figure}[t]
\centering
\includegraphics[width=0.48\textwidth]{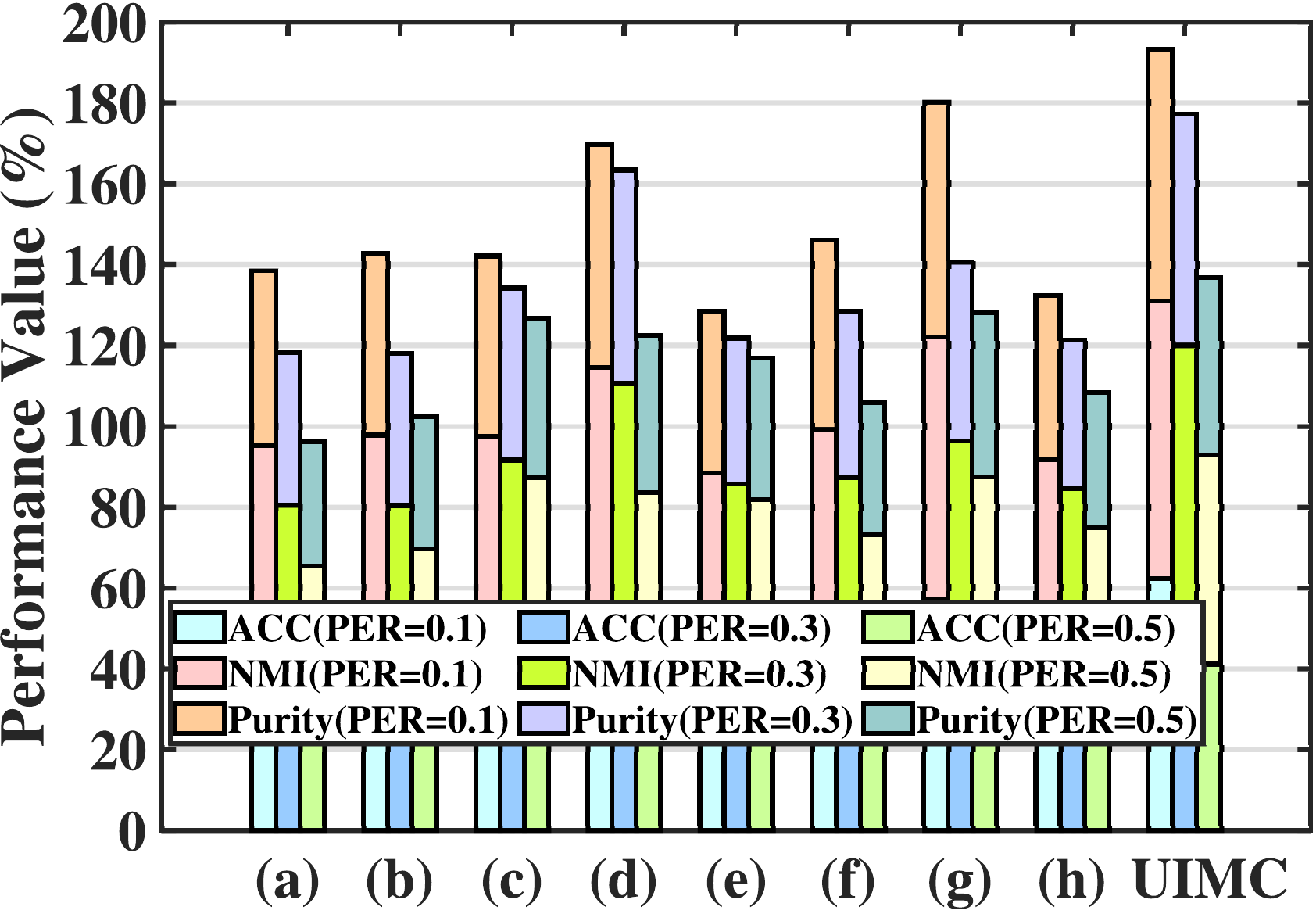}
\caption{Balanced incomplete multi-view clustering results, where (a) is BSV, (b) is Concat, (c) is  PVC, (d) is IMG, (e) is MIC, (f) is DAIMC, (g) is IMSC$\_$AGL, (h) is UEAF.}
\label{fig:UIMC_b}
\end{figure}
\subsection{Complexity Analysis}
From Section \ref{subsection:opt}, the major computational costs of our proposed UIMC mainly come from the operations like matrix inverse, SVD, and eigenvalue decomposition. Therefore,
Steps 3, 7, and 8 are the main computational costs.
For Steps 3, 7, and 8, they have the same maximum computational complexity $O(k_v^3)$.
Therefore, the whole computational complexity of UIMC is about $O(i\sum_{v=1}^{n_v}k_v^3)$, where $i$ is the iteration number, $n_v$ is the view number, and $k_v$ is the number of presented instances in view $v$. To quantify the computational complexity, we report the time cost under different PERs in Table~\ref{runtime}.
\section{Performance Evaluation}
\label{section:exp}
We first illustrate the clustering performance of the proposed UIMC, then verify UIMC's convergence, and finally analyze UIMC's parameter sensitivity.
\begin{figure*}[t]
\centering
\includegraphics[width=\textwidth]{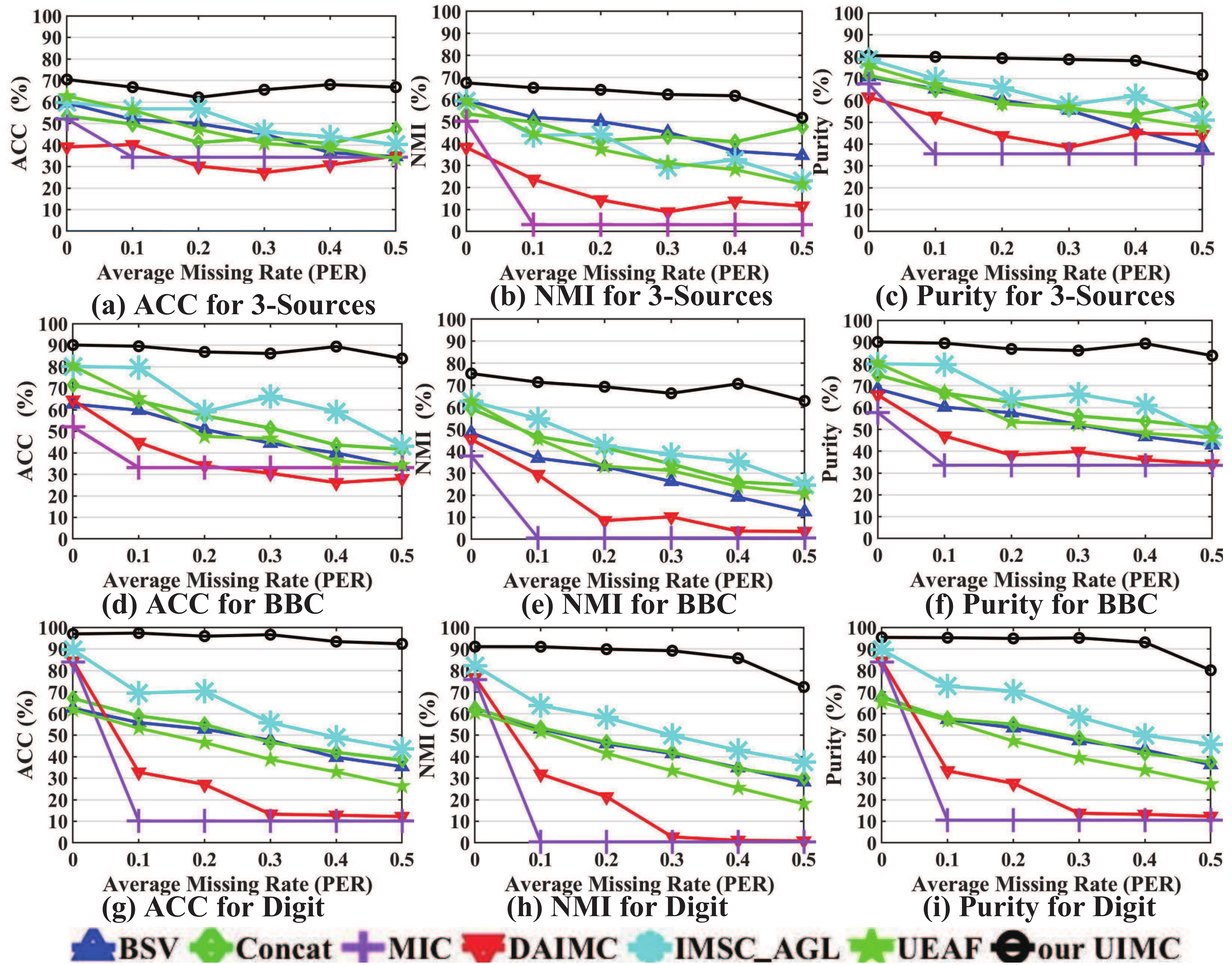}
\caption{Unbalanced incomplete multi-view clustering results.}
\label{fig:UIMC_unb}
\end{figure*}
\begin{figure*}[t]
\centering
\subfigure[ACC (BI) for 3-Sources]{\label{fig:zhuzhuangtu_acc} \includegraphics[width=0.3\textwidth]{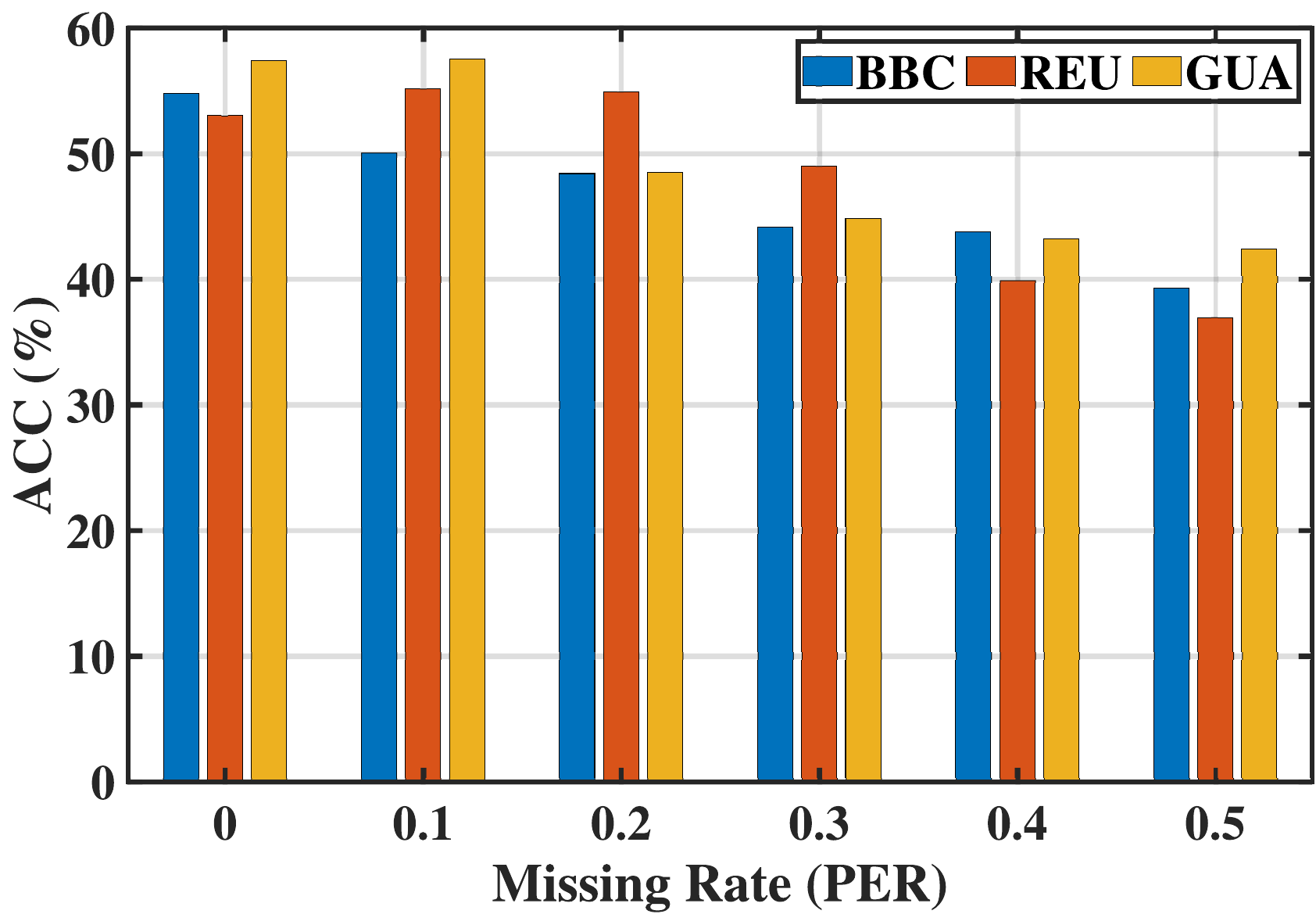} }
\subfigure[NMI (BI) for 3-Sources]{\label{fig:zhuzhuangtu_nmi} \includegraphics[width=0.3\textwidth]{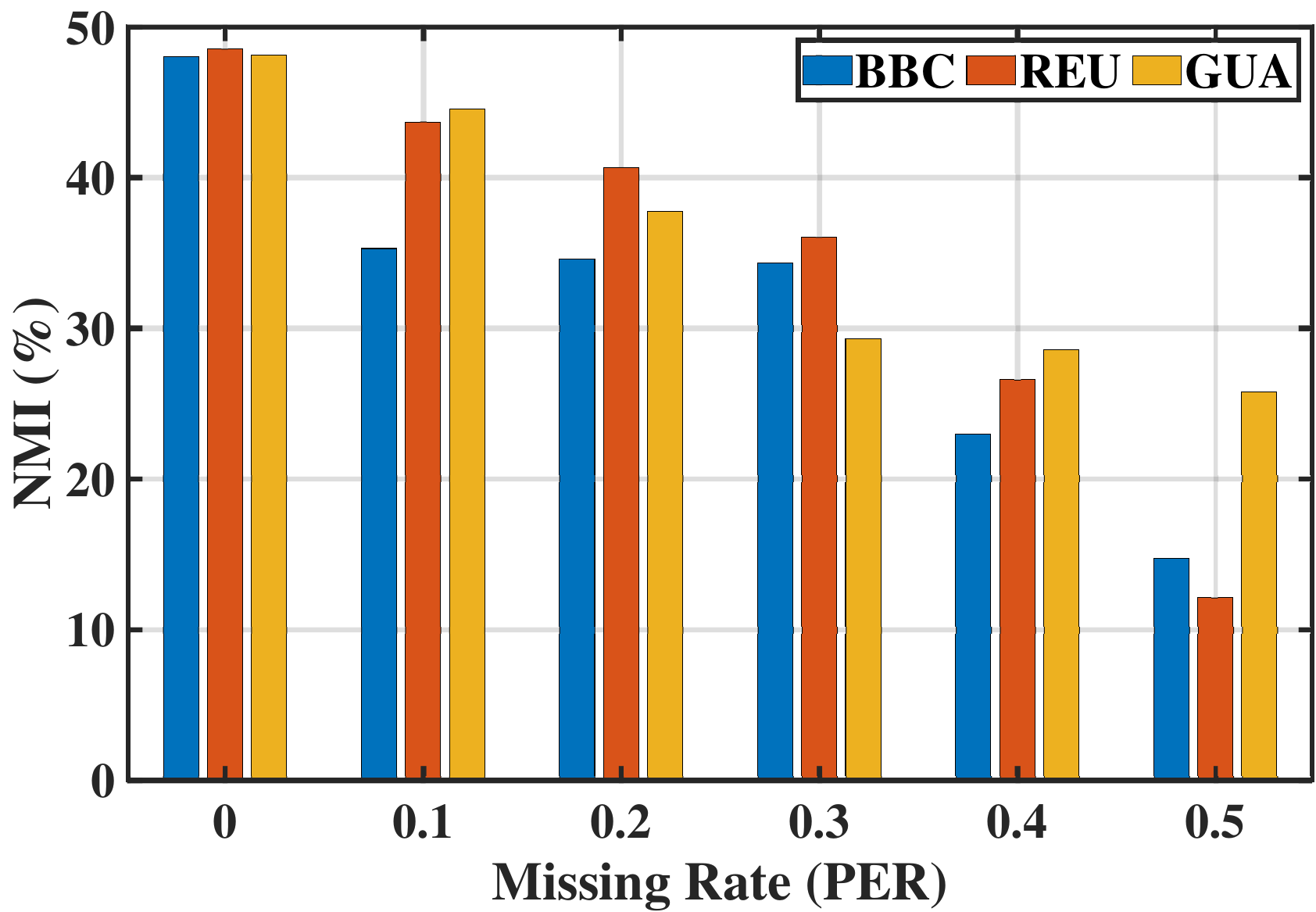} }
\subfigure[Purity (BI)  for 3-Sources]{\label{fig:zhuzhuangtu_pur} \includegraphics[width=0.3\textwidth]{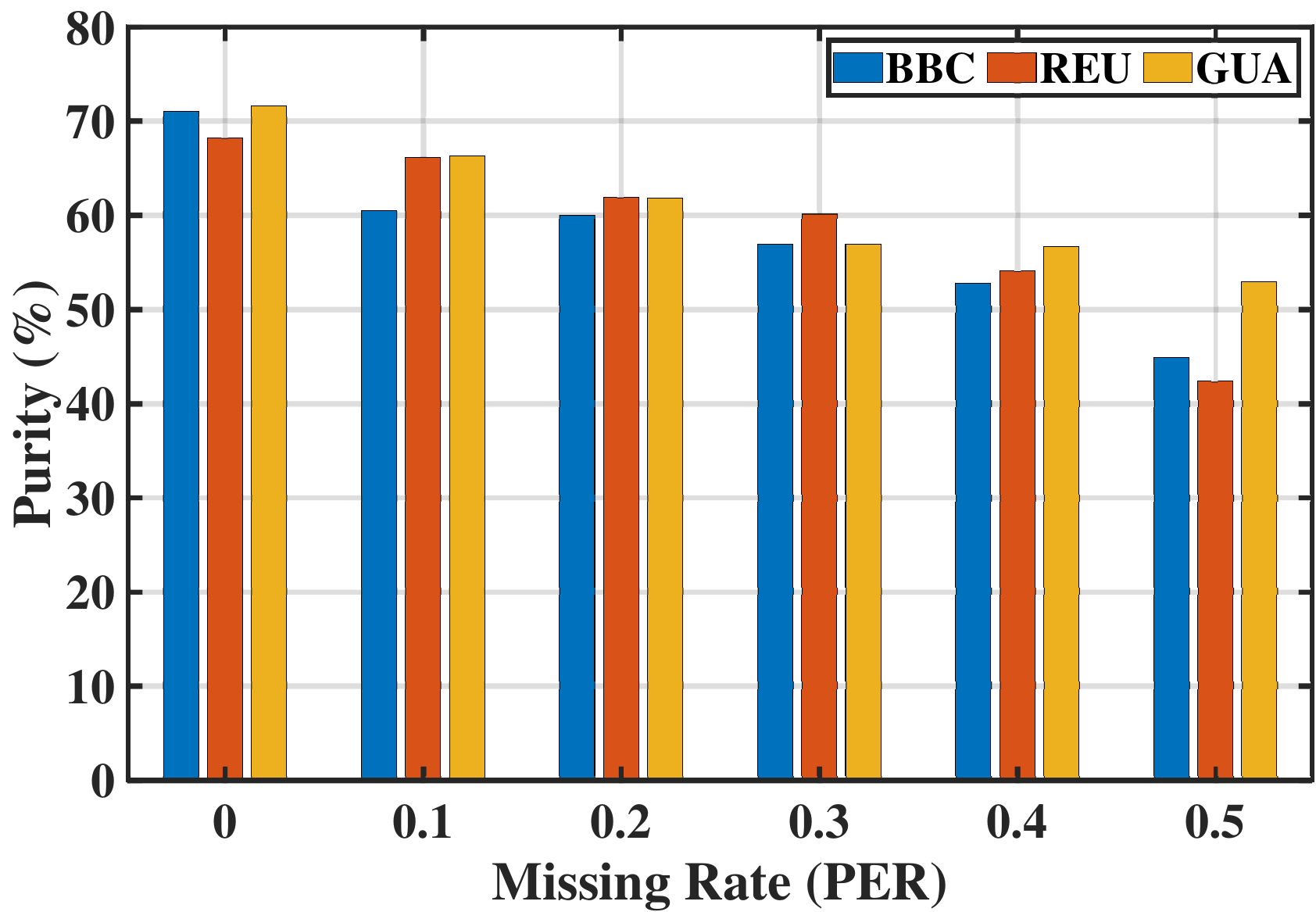}}
\subfigure[ACC (BI) for BBC]{\label{fig:zhuzhuangtu_bbc_acc} \includegraphics[width=0.3\textwidth]{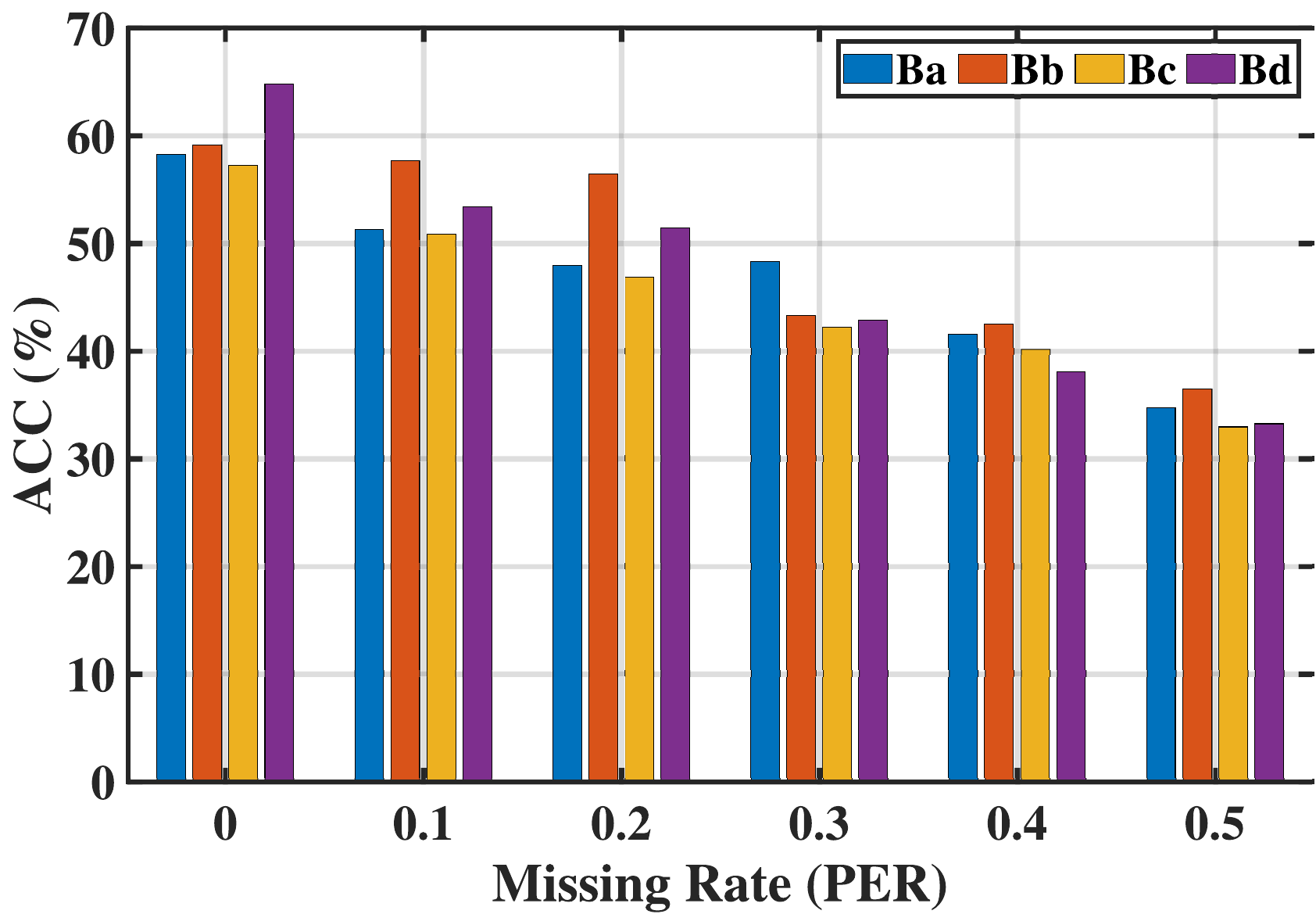} }
\subfigure[NMI (BI) for BBC]{\label{fig:zhuzhuangtu_bbc_nmi} \includegraphics[width=0.3\textwidth]{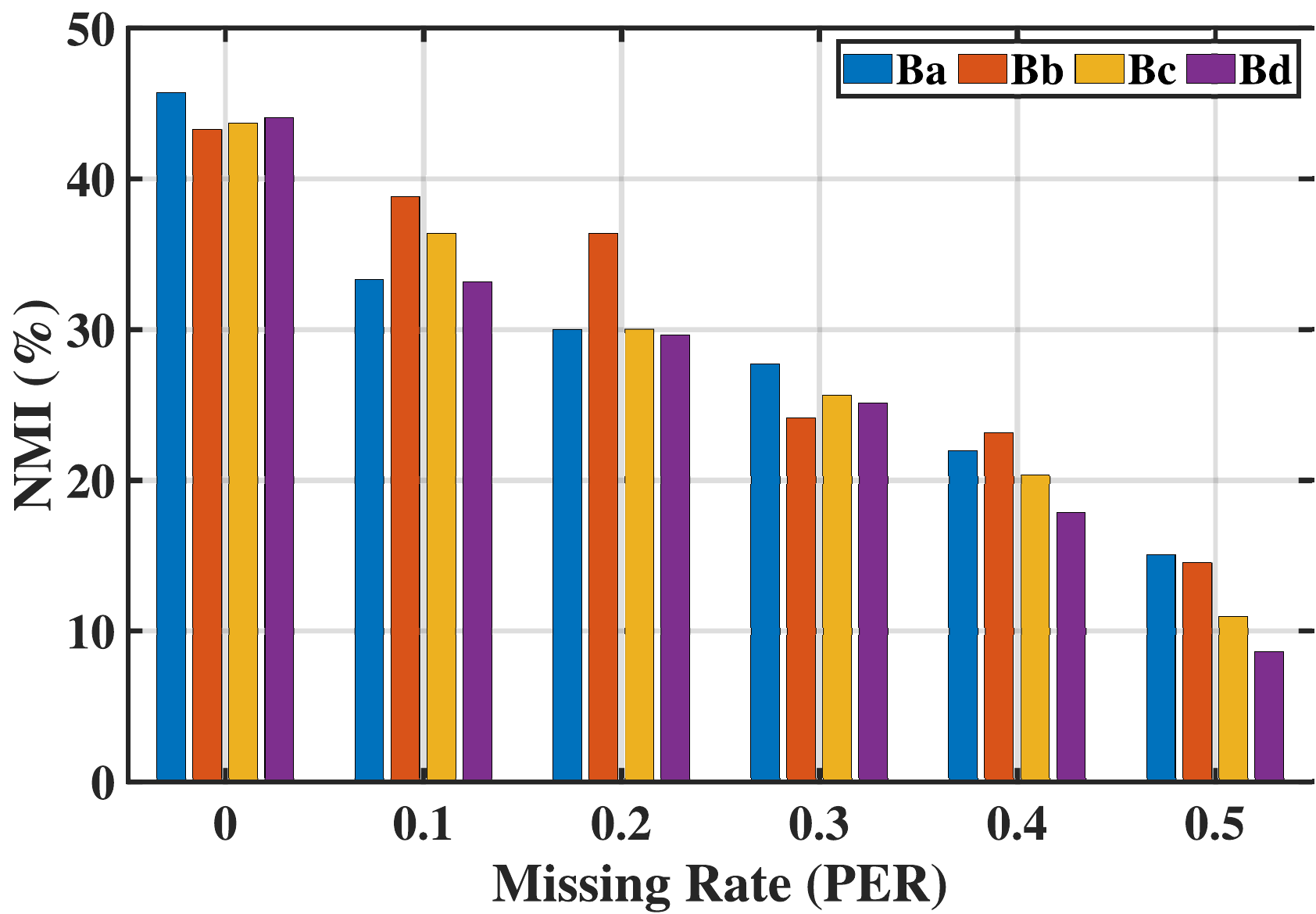} }
\subfigure[Purity (BI) for BBC]{\label{fig:zhuzhuangtu_bbc_pur} \includegraphics[width=0.3\textwidth]{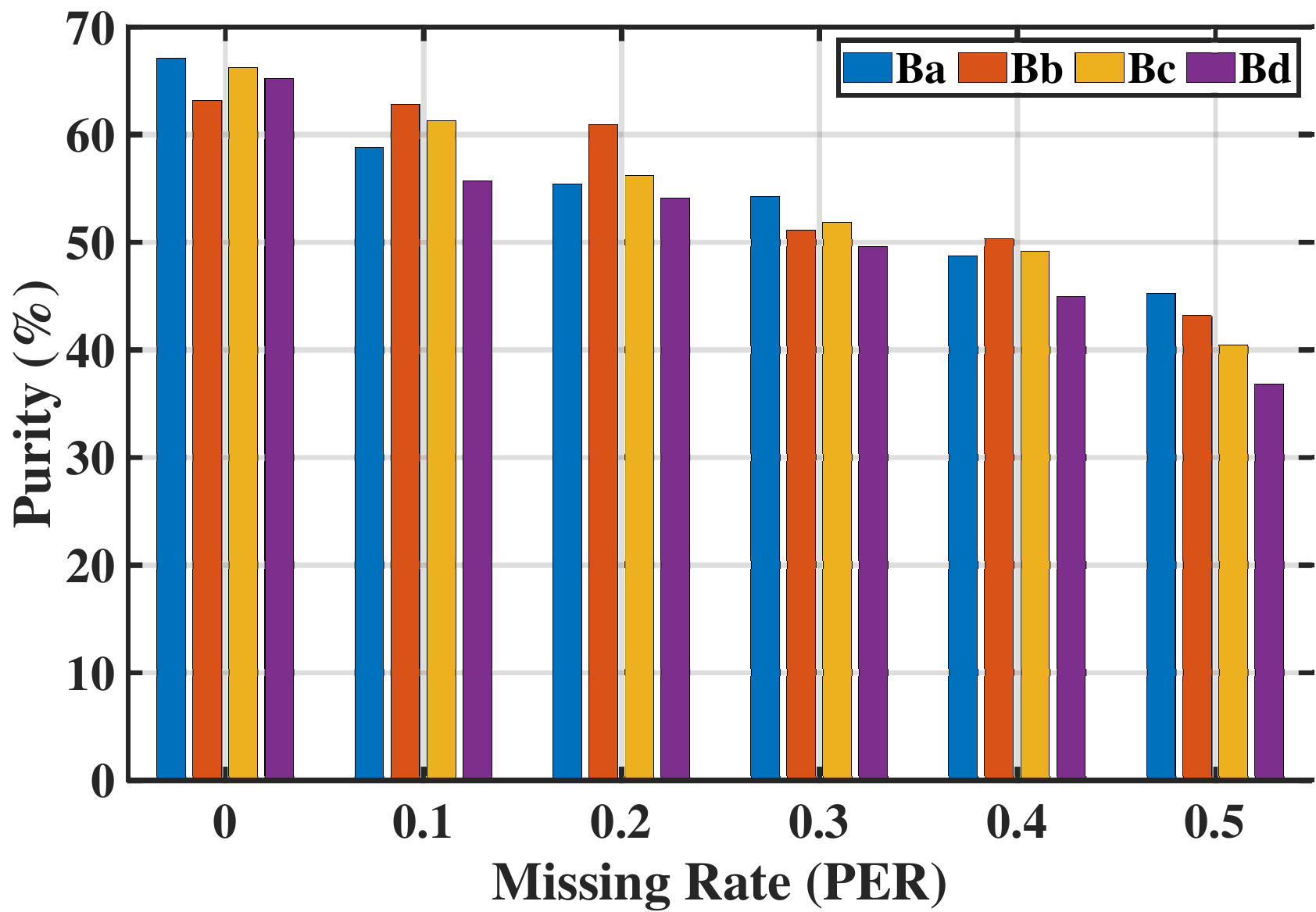}}
\subfigure[ACC (BI) for Digit]{\label{fig:zhuzhuangtu_digit_acc} \includegraphics[width=0.3\textwidth]{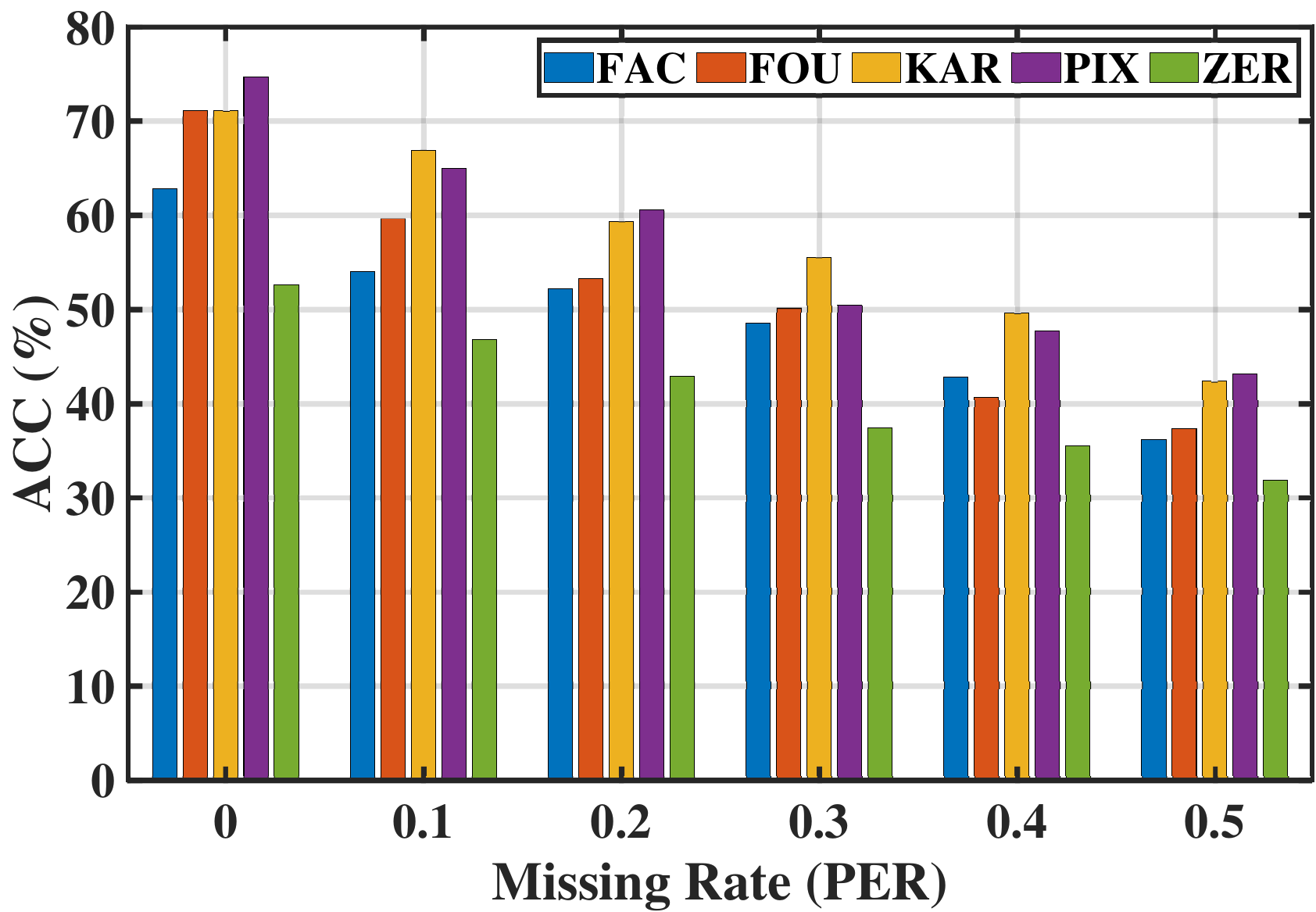} }
\subfigure[NMI (BI) for Digit]{\label{fig:zhuzhuangtu_digit_nmi} \includegraphics[width=0.3\textwidth]{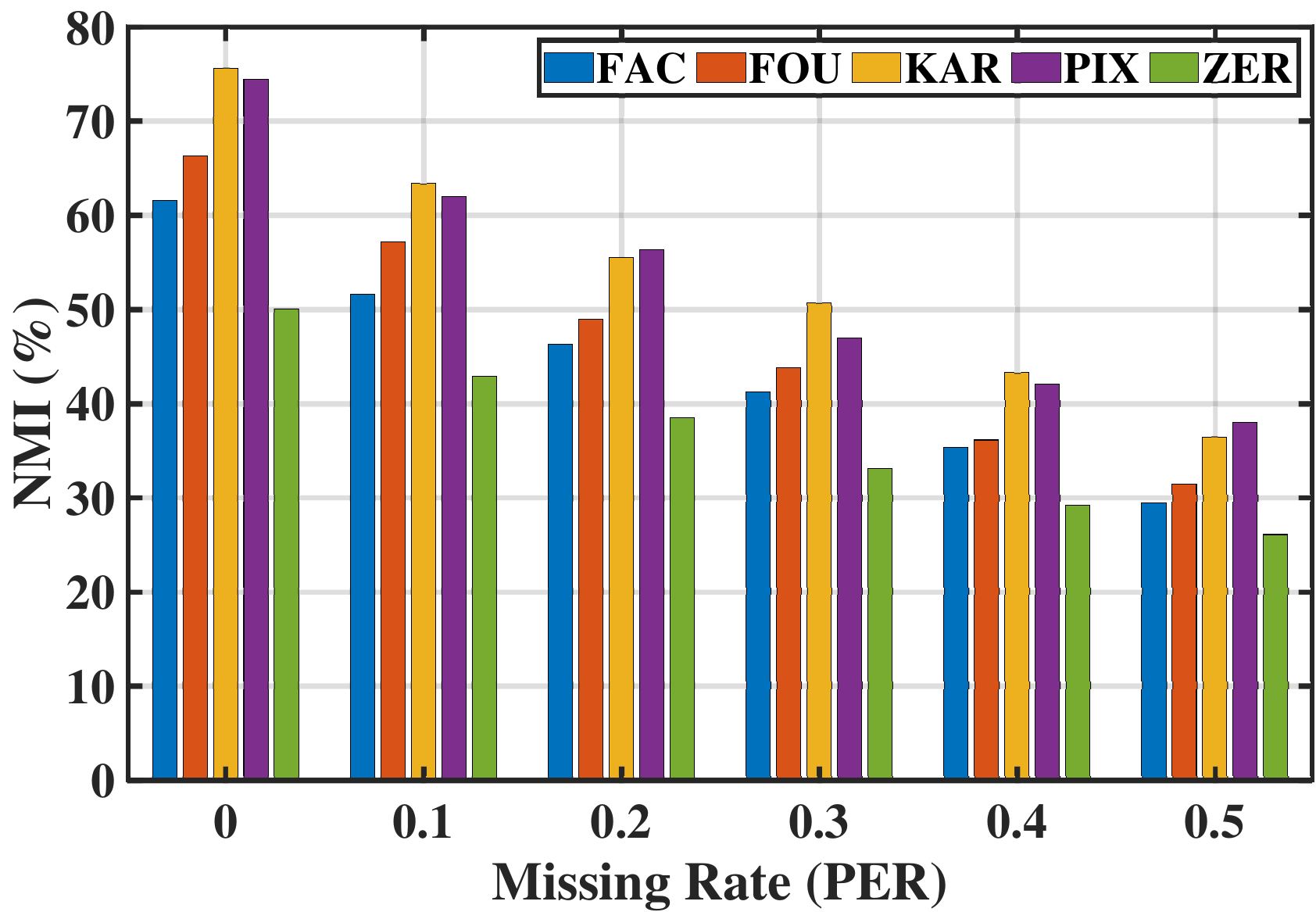} }
\subfigure[Purity (BI) for Digit]{\label{fig:zhuzhuangtu_digit_pur} \includegraphics[width=0.3\textwidth]{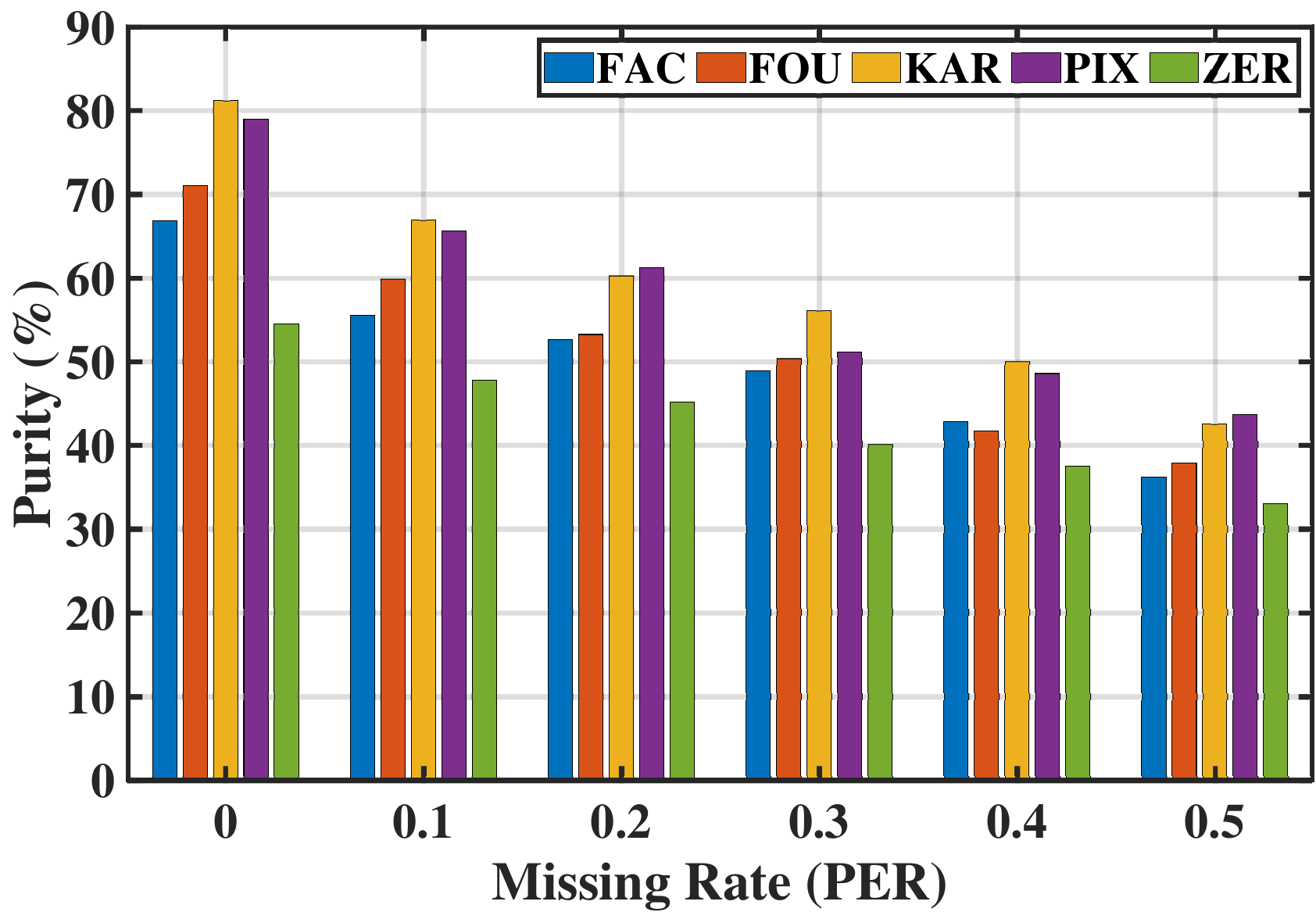}}
\caption{View availability comparison on balanced incomplete multi-view clustering, where ``BI'' means ``balanced incomplete multi-view clustering''.}
\label{fig:dangeshitu_ba}
\end{figure*}
\begin{figure*}[t]
\centering
\subfigure[ACC (UI) for 3-Sources]{\label{fig:unbanlanced_zhu_acc} \includegraphics[width=0.235\textwidth]{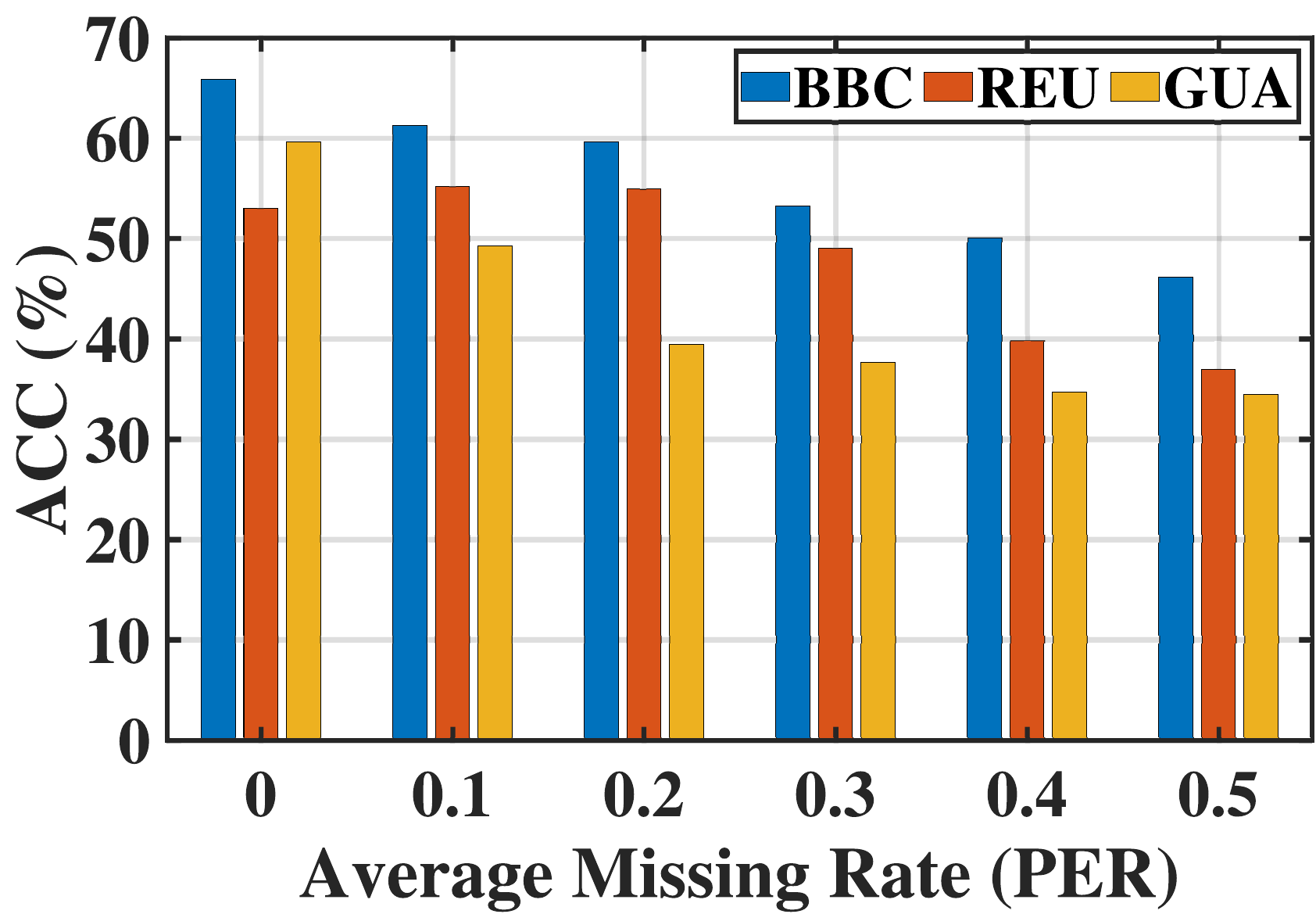} }
\subfigure[NMI (UI) for 3-Sources]{\label{fig:unbanlanced_zhu_nmi} \includegraphics[width=0.235\textwidth]{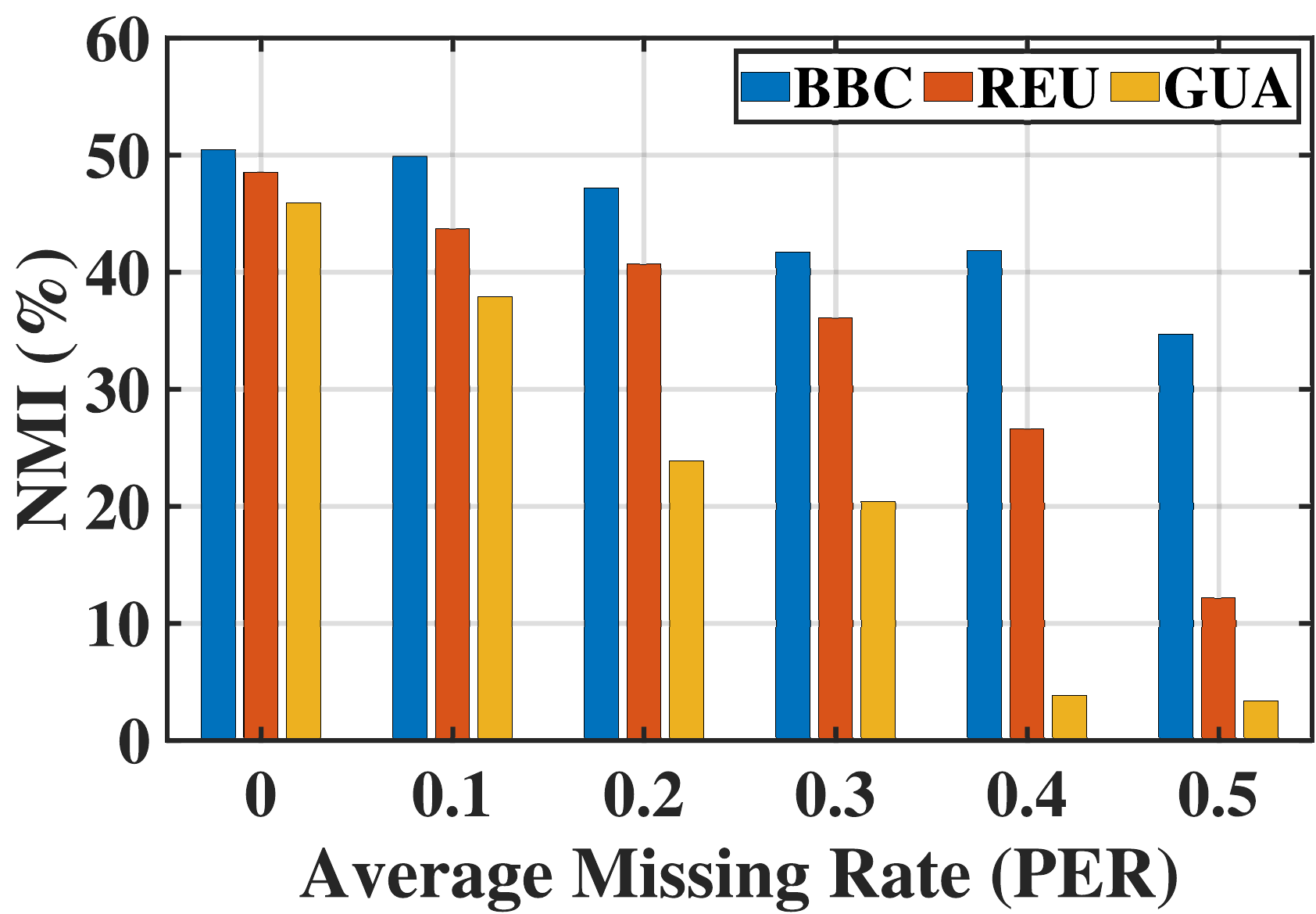} }
\subfigure[Purity (UI) for 3-Sources]{\label{fig:unbanlanced_zhu_pur} \includegraphics[width=0.235\textwidth]{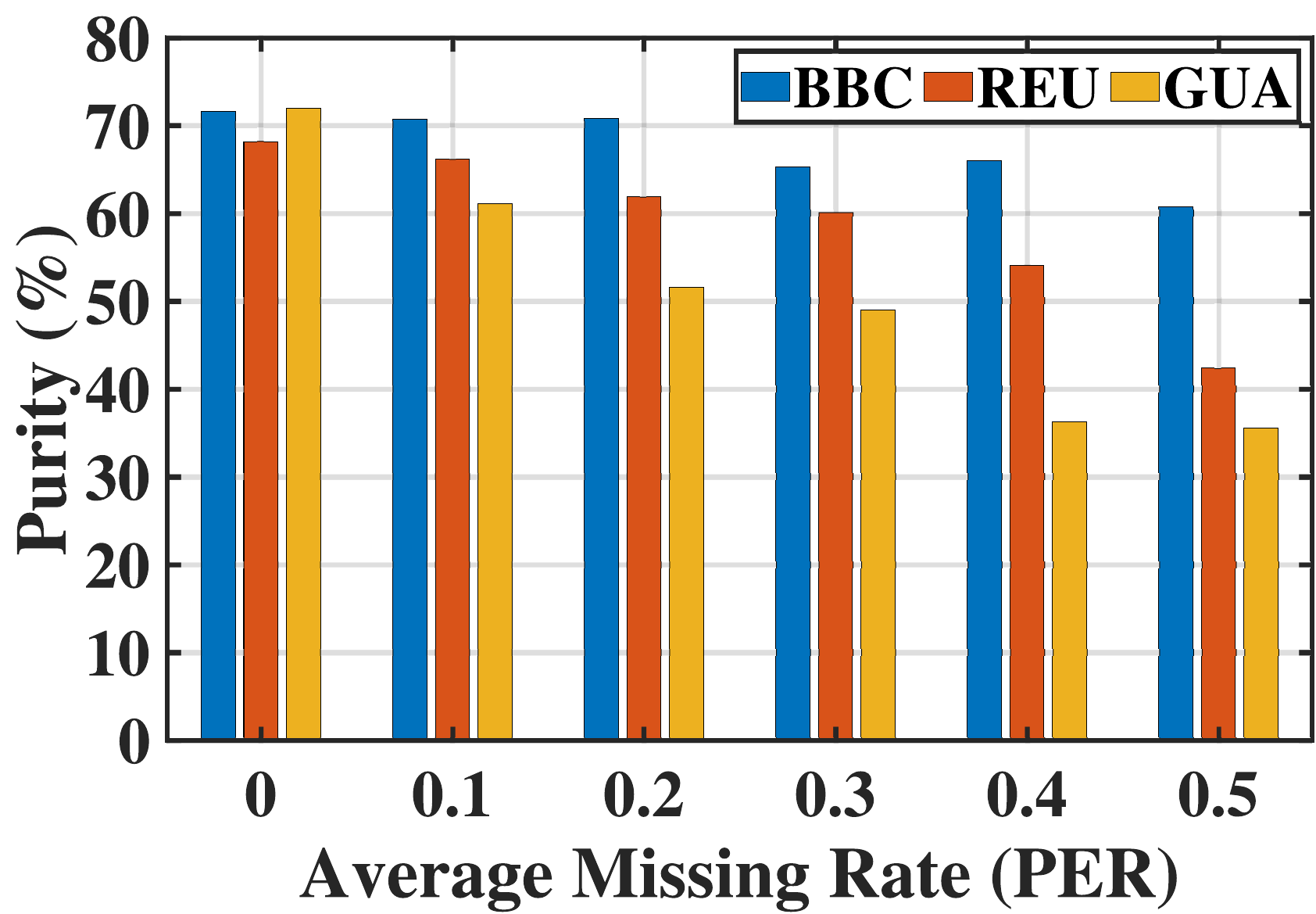}}
\subfigure[View completeness for 3-Sources]{\label{fig:unbanlanced_zhu_wanzhengxing} \includegraphics[width=0.235\textwidth]{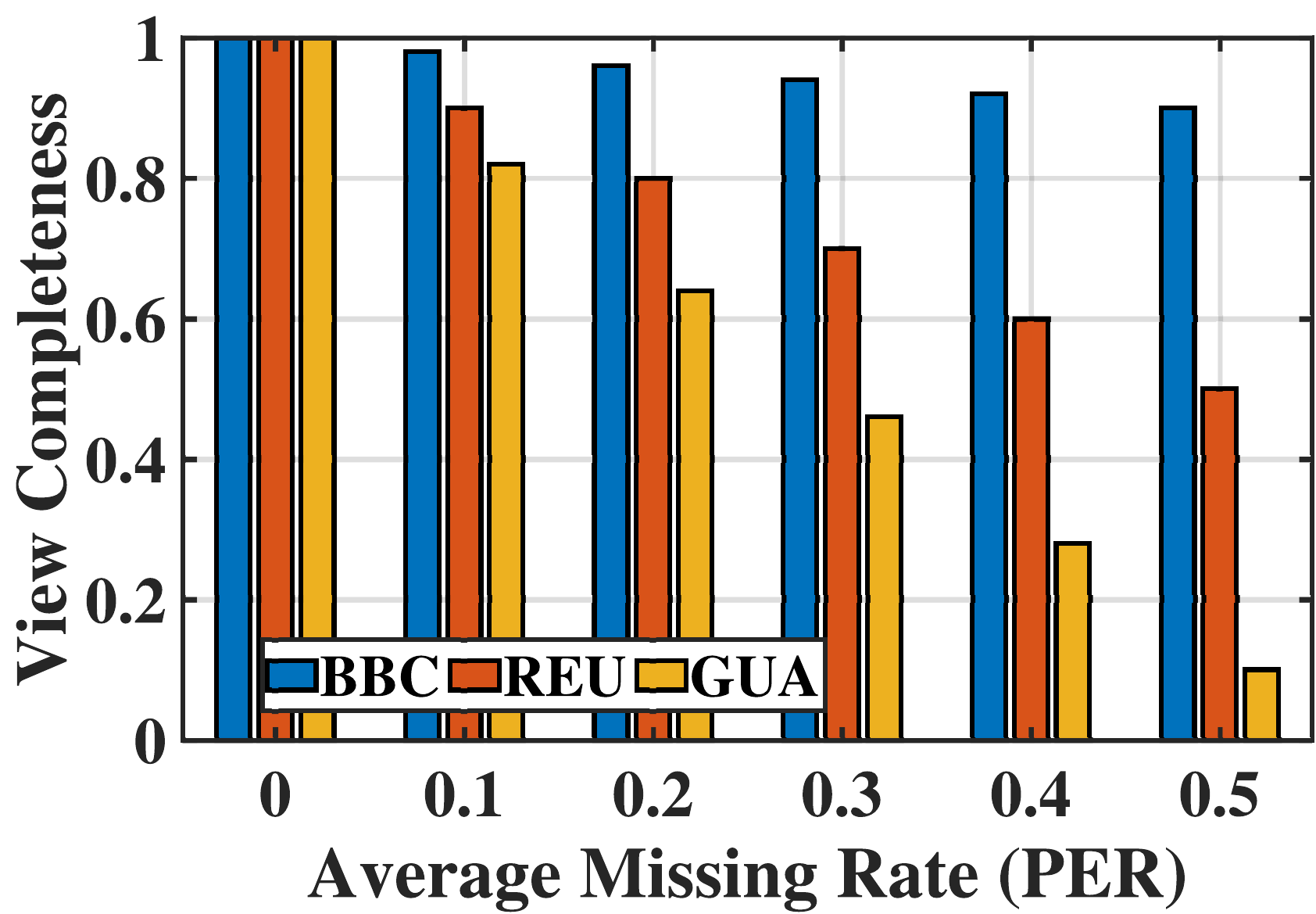}}
\subfigure[ACC (UI) for BBC]{\label{fig:unbanlanced_zhu_bbc_acc} \includegraphics[width=0.235\textwidth]{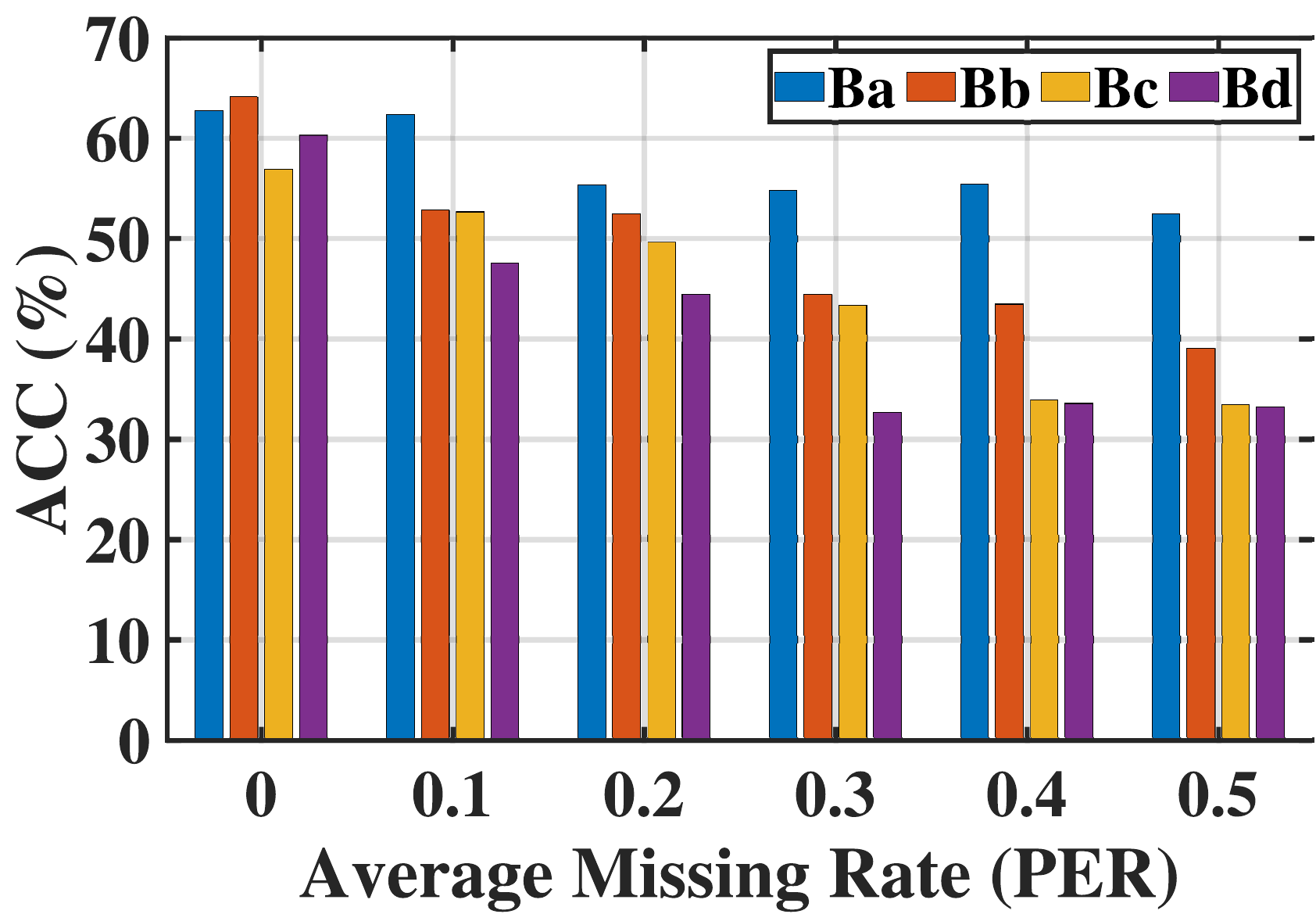} }
\subfigure[NMI (UI) for BBC]{\label{fig:unbanlanced_zhu_bbc_nmi} \includegraphics[width=0.235\textwidth]{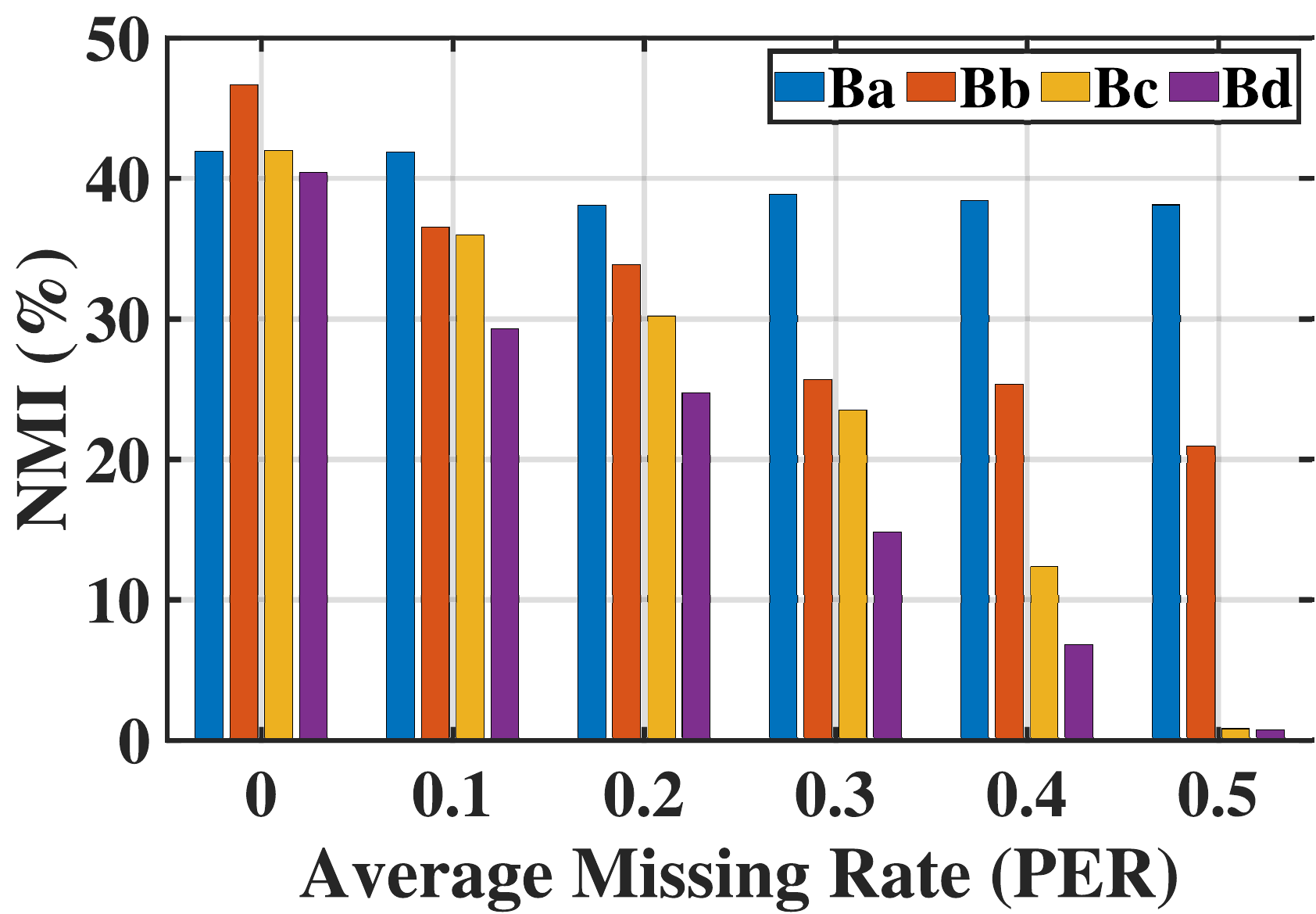} }
\subfigure[Purity (UI) for BBC]{\label{fig:unbanlanced_zhu_bbc_pur} \includegraphics[width=0.235\textwidth]{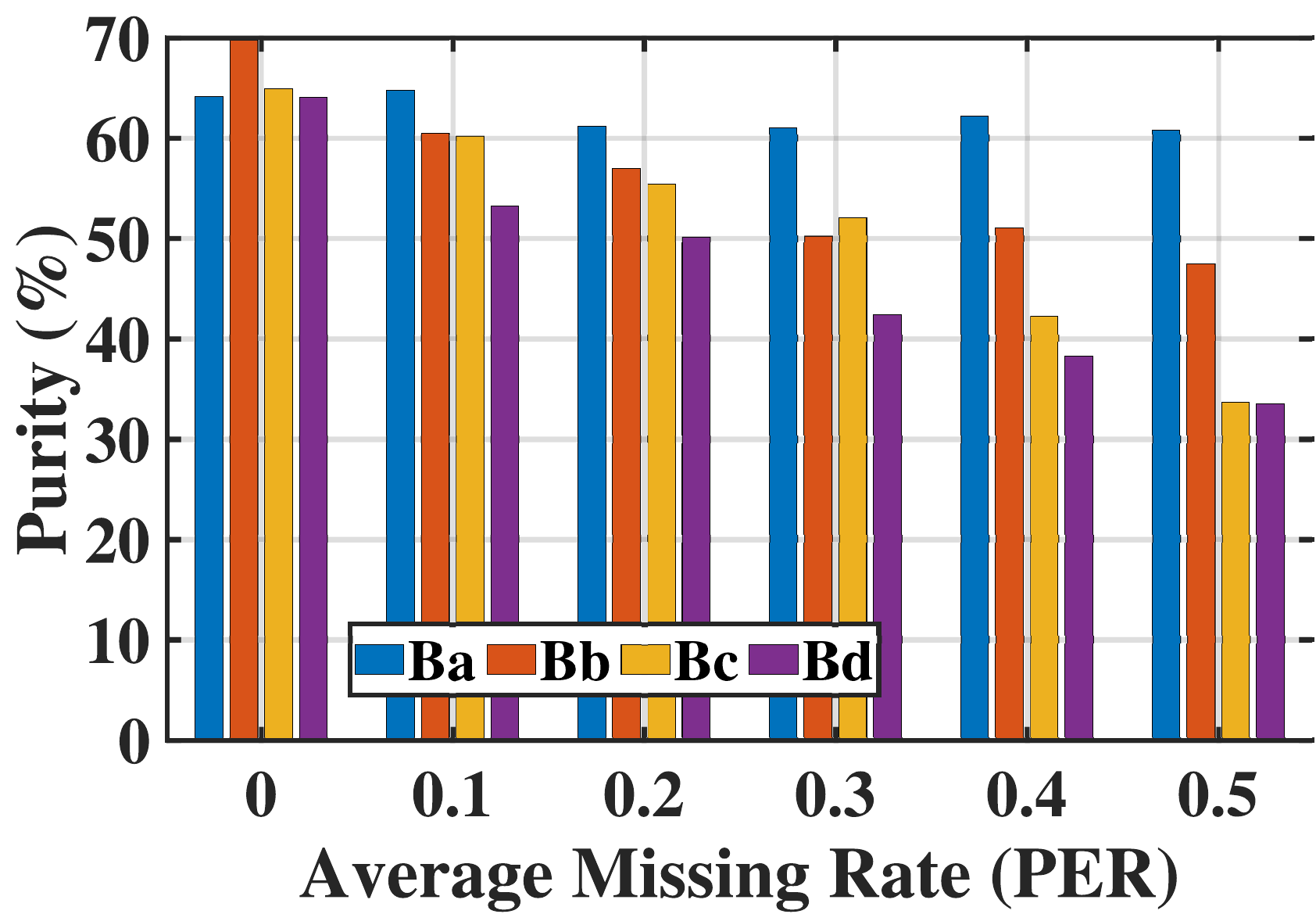}}
\subfigure[View completeness for BBC]{\label{fig:unbanlanced_zhu_bbc_wanzhengxing} \includegraphics[width=0.235\textwidth]{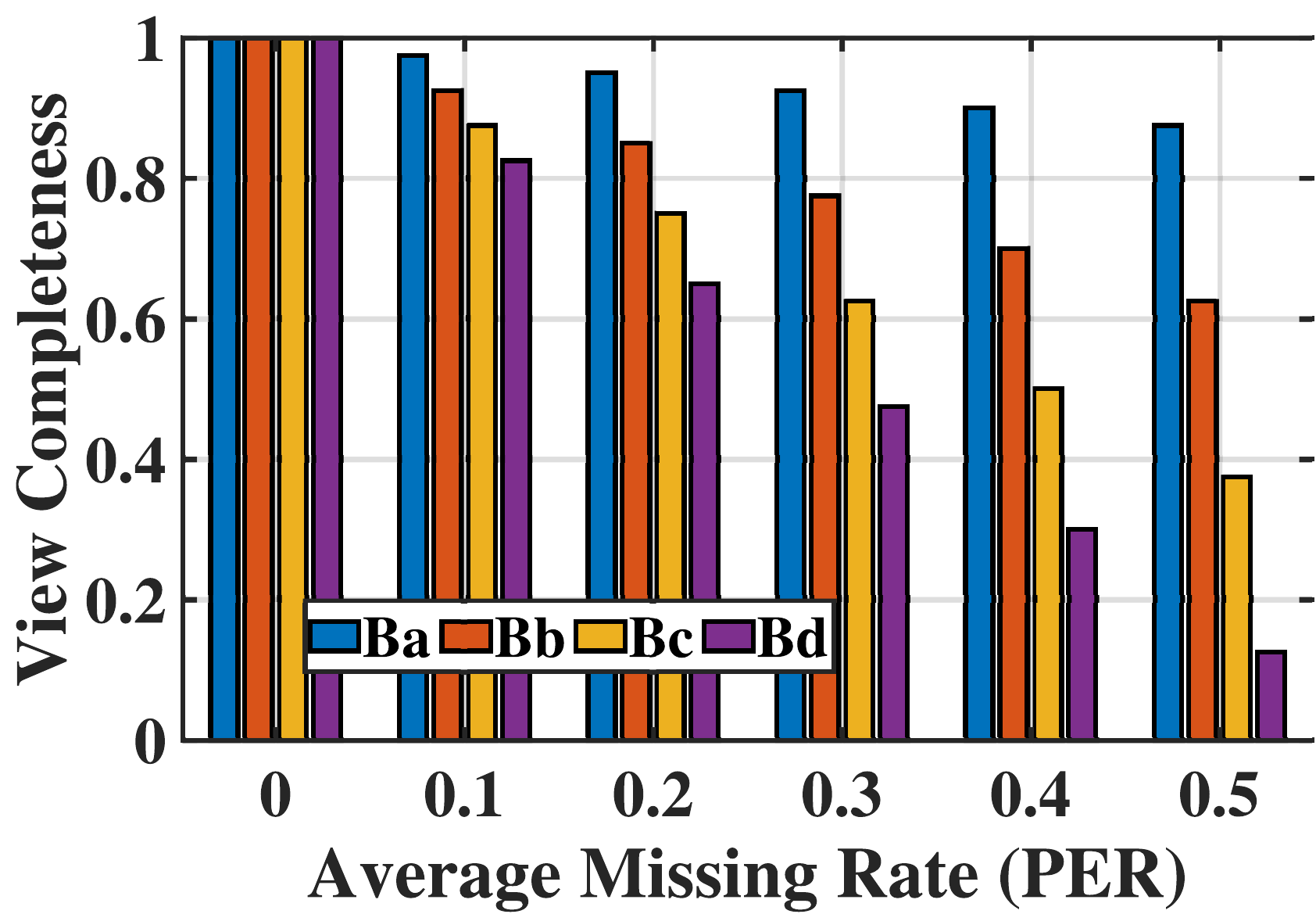}}
\caption{View availability comparison on unbalanced incomplete multi-view clustering, where ``UI'' means ``unbalanced incomplete multi-view clustering'' and ``View completeness'' represents the proportion of  presented instances in a view.}
\label{fig:dangeshitu_unba}
\end{figure*}
\subsection{Datasets}\label{subsection:dataset}
We conduct experiments on four real-world multi-view datasets (shown in Table~\ref{dataset}) as follows: BUAA dataset \cite{huang2012buaa}, 3-Sources dataset\footnote{\url{http://erdos.ucd.ie/datasets/3sources.html}.}, BBC dataset\footnote{\url{http://mlg.ucd.ie/datasets/segment.html}.}, and Digit dataset\footnote{\url{http://archive.ics.uci.edu/ml/datasets/Multiple+Features}.}.

%
%
%
%
1) BUAA-visnir face dataset (\textbf{BUAA}): BUAA has 90 visual images and 90 near infrared images (i.e., 180 instances), which are categorized into 10 classes (i.e., 10 clusters). Each image is described by two views: VIS and NIR.

2) 3 Sources dataset (\textbf{3 Sources}): 3-Sources contains 948 news articles collected from three online news sources (i.e., 3 views): BBC, Reuters (REU), and The Guardian (GUA). We select a subset with 169 articles (i.e., 169 instances), which are categorized into 6 topical labels (i.e., 6 clusters).

3) BBC dataset (\textbf{BBC}): BBC is collected from the BBC news website consisting of 685 documents (i.e., 685 instances). Each document was split into four segments (i.e., View Ba, View Bb, View Bc, and View Bd) and was manually annotated with one of five topical labels (i.e., 5 clusters).

4) Handwritten digit dataset (\textbf{Digit}): Digit contains 2000 handwritten numerals (i.e., 2000 instances) collected from the following views:
(i) 216 profile correlations (FAC), (ii) 76 Fourier coefficients of the character shapes (FOU), (iii) 64 Karhunen-Love coefficients (KAR), (iv) 240 pixel averages in $2\times 3$ windows (PIX), (v) 47 Zernike moments (ZER).

\noindent\textbf{Data preprocessing}: To obtain incomplete multi-view data for clustering, we need to preprocess the real data. Following \cite{hu2018doubly}, we randomly delete some instances from each view to obtain incomplete views.
For an incomplete multi-view dataset, we denote its average missing rate\footnote{The average missing rate is the average of the missing rates of all views.} as PER, and we set different PERs from 0 to 0.5.
To describe the missing rate of each view, we define an average missing rate vector $\bm{r}$=$[0,0.1,0.2,\ldots,0.5]$, which represents a series of average missing rates (from PER=0 to PER=0.5). For an unbalanced incomplete multi-view dataset, different views have distinct missing rate vectors, whose average is average missing rate vector $\bm{r}$. The unbalanced incompleteness includes many cases, and we use a representative case for each dataset.
The used missing rate vector of each view is shown in Table~\ref{datashezhi}. For example, for the 3-Sources dataset, when the average missing rate PER=0.1, the missing rate of GUA is $1.8\times 0.1$=$0.18$.
\subsection{Compared Methods}
We compare our proposed UIMC with the following eight state-of-the-art methods, which are the most relevant to our work:

(a) \textbf{BSV} (Best Single View)~\cite{zhao2016incomplete} first fills in the missing views with the average of instances in the corresponding view, then performs K-means clustering on each view separately. Finally, we report the best clustering result.

(b) \textbf{Concat}~\cite{zhao2016incomplete}, same as BSV, first fills in all the missing data, then concatenates all views into one single view, and finally performs K-means clustering on the single view.

(c) \textbf{PVC}~\cite{li2014partial} establishes a latent subspace to align two incomplete views.

(d) \textbf{IMG}~\cite{zhao2016incomplete} extends PVC by introducing manifold learning.

(e) \textbf{MIC}~\cite{shao2015multiple} extends MultiNMF via weighted NMF and $L_{2,1}$ regularization.

(f) \textbf{DAIMC}~\cite{hu2018doubly} extends MIC via weighted semiNMF and $L_{2,1}$-norm regularization regression.

(g) \textbf{IMSC$\_$AGL}~\cite{wen2018incompleteb} learns the common representation by exploiting  graph learning and spectral clustering.

(h) \textbf{UEAF}~\cite{wen2019unified} performs clustering by aligning the unified common embedding.

BSV and Concat are the baseline methods in our experiments.
Note that all results of compared methods are produced by released codes, some of which may be inconsistent with published information due to different pretreatment processes. All the codes in our experiments are implemented in MATLAB R2020b and run on a Windows 10 machine with $3.30$ GHz E3-1225 CPU, $64$ GB main memory.
We perform the following two types of experiments:

\noindent 1) Balanced incomplete multi-view clustering on the BUAA dataset: We compare UIMC with all the above methods.

\noindent 2) Unbalanced incomplete multi-view clustering on 3-Sources, BBC and Digit datasets: Note that MIC, DAIMC, IMSC$\_$AGL, and UEAF cannot cluster unbalanced incomplete multi-view datasets. For each dataset, similar to Concat, we concatenate all incomplete views into one view, and perform these methods on the single view. Since PVC and IMG are based on two views (neither one view nor more than two views), we do not compare UIMC with PVC and IMG here. For three parameters in UIMC, we set $\alpha$=$10^{-2}$, $\beta$=$10^{5}$, and $\eta$=$10^{-1}$.

\noindent\textbf{Evaluation metrics}: Following \cite{li2014partial}, we repeat each experiment 10 times and report the average clustering results.
Following~\cite{wen2018incompleteb,wen2019unified}, we evaluate the results by three popular clustering evaluation metrics: accuracy (ACC), normalized mutual information (NMI), and Purity. For these evaluation metrics, the larger value represents the better clustering performance.
%

\begin{table}[ht]
\centering
\scalebox{0.88}{
\begin{tabular}{c|c|c|cc}
\cline{1-5}
\multirow{2}*{Dataset}&\multirow{2}*{PER}&\multirow{2}*{View}&\multicolumn{2}{c}{Weight}\\\cline{4-5}
~&~&~&UEAF&Our UIMC\\\cline{1-5}
\multirow{18}*{3-Sources}&\multirow{3}*{0}&BBC&0.3316&0.3333\\\cline{3-5}
~&~&REU&0.3335&0.3333\\\cline{3-5}
~&~&GUA&0.3350&0.3333\\\cline{2-5}
~&\multirow{3}*{0.1}&BBC&0.3113&0.3597\\\cline{3-5}
~&~&REU&0.3333&0.3267\\\cline{3-5}
~&~&GUA&0.3554&0.3126\\\cline{2-5}
~&\multirow{3}*{0.2}&BBC&0.1155& 0.3976\\\cline{3-5}
~&~&REU&0.3715&0.3268\\\cline{3-5}
~&~&GUA&0.5129&0.2756\\\cline{2-5}
~&\multirow{3}*{0.3}&BBC&0.0634&0.4344\\\cline{3-5}
~&~&REU&0.1052&0.3470\\\cline{3-5}
~&~&GUA&0.8314&0.2186\\\cline{2-5}
~&\multirow{3}*{0.4}&BBC&0.0365&0.4952\\\cline{3-5}
~&~&REU&0.0628&0.3492\\\cline{3-5}
~&~&GUA&0.9007&0.1556\\\cline{2-5}
~&\multirow{3}*{0.5}&BBC&0.0060&0.6016\\\cline{3-5}
~&~&REU&0.0106&0.3333\\\cline{3-5}
~&~&GUA&0.9834&0.0650\\\cline{1-5}
\multirow{24}*{BBC}&\multirow{4}*{0}&Ba&0.2498&0.2500\\\cline{3-5}
~&~&Bb&0.2500&0.2500 \\\cline{3-5}
~&~&Bc&0.2499&0.2500\\\cline{3-5}
~&~&Bd&0.2503&0.2500\\\cline{2-5}
~&\multirow{4}*{0.1}&Ba&0.2391&0.2747\\\cline{3-5}
~&~&Bb&0.2450&0.2604\\\cline{3-5}
~&~&Bc&0.2518&0.2480 \\\cline{3-5}
~&~&Bd&0.2641&0.2168\\\cline{2-5}
~&\multirow{4}*{0.2}&Ba&0.0901& 0.2941 \\\cline{3-5}
~&~&Bb&0.1901&0.2619 \\\cline{3-5}
~&~&Bc&0.3481&0.2377\\\cline{3-5}
~&~&Bd&0.3716&0.2064\\\cline{2-5}
~&\multirow{4}*{0.3}&Ba&0.0618&0.3235\\\cline{3-5}
~&~&Bb&0.0853&0.2756 \\\cline{3-5}
~&~&Bc&0.1831&0.2220 \\\cline{3-5}
~&~&Bd&0.6697&0.1789\\\cline{2-5}
~&\multirow{4}*{0.4}&Ba&0.0398&0.3730\\\cline{3-5}
~&~&Bb&0.0536&0.2978\\\cline{3-5}
~&~&Bc&0.0811&0.2171\\\cline{3-5}
~&~&Bd&0.8255&0.1122\\\cline{2-5}
~&\multirow{4}*{0.5}&Ba&0.0198&0.4371\\\cline{3-5}
~&~&Bb&0.0280&0.3134 \\\cline{3-5}
~&~&Bc&0.0401&0.1984 \\\cline{3-5}
~&~&Bd&0.9120&0.0510\\\cline{1-5}
\multirow{30}*{Digit}&\multirow{5}*{0}&FAC&0.7537&0.2000\\\cline{3-5}
~&~&FOU&0.0720&0.2000\\\cline{3-5}
~&~&KAR&0.0262&0.2000\\\cline{3-5}
~&~&PIX&0.0519&0.2000\\\cline{3-5}
~&~&ZER&0.0961&0.2000\\\cline{2-5}
~&\multirow{5}*{0.1}&FAC&0.7368&0.2175\\\cline{3-5}
~&~&FOU&0.0745&0.2046\\\cline{3-5}
~&~&KAR&0.0275&0.2000\\\cline{3-5}
~&~&PIX&0.0556&0.1978\\\cline{3-5}
~&~&ZER&0.1055&0.1801\\\cline{2-5}
~&\multirow{5}*{0.2}&FAC&0.7142&0.2383\\\cline{3-5}
~&~&FOU&0.0741&0.2141\\\cline{3-5}
~&~&KAR&0.0289&0.1992\\\cline{3-5}
~&~&PIX&0.0594&0.1861\\\cline{3-5}
~&~&ZER&0.1234&0.1623\\\cline{2-5}
~&\multirow{5}*{0.3}&FAC&0.1136&0.2671\\\cline{3-5}
~&~&FOU&0.0998&0.2222 \\\cline{3-5}
~&~&KAR&0.0510&0.2045 \\\cline{3-5}
~&~&PIX&0.1170&0.1755\\\cline{3-5}
~&~&ZER&0.6186&0.1307\\\cline{2-5}
~&\multirow{5}*{0.4}&FAC&0.0556&0.2994 \\\cline{3-5}
~&~&FOU&0.0578&0.2374\\\cline{3-5}
~&~&KAR&0.0408& 0.1952\\\cline{3-5}
~&~&PIX&0.1080&0.1656\\\cline{3-5}
~&~&ZER&0.7377&0.1024\\\cline{2-5}
~&\multirow{5}*{0.5}&FAC&0.0279&0.3487 \\\cline{3-5}
~&~&FOU&0.0312&0.2521\\\cline{3-5}
~&~&KAR&0.0262&0.2006\\\cline{3-5}
~&~&PIX&0.0653&0.1518\\\cline{3-5}
~&~&ZER&0.8494&0.0468\\\cline{1-5}
\end{tabular}
}
\caption{Weight comparison between UEAF and UIMC.}
\label{quanzhongduibi}
\end{table}
\begin{table}[t]
\centering
\scalebox{0.95}{
\begin{tabular}{ccc|ccccccc}
\hline
\multicolumn{3}{c|}{Missing rate of each view}&\multirow{2}*{Method}&\multirow{2}*{ACC (\%)}&\multirow{2}*{NMI (\%)}&\multirow{2}*{Purity (\%)}\\\cline{1-3}
BBC&REU&GUA& ~  & ~ & ~ &~\\\hline
0.5& 0.5& 0.5&BSV       & 34.56  & 6.64    & 38.34  \\\hline
0.5& 0.5& 0.5 &Concat    & 47.46 & 34.13    & 58.34  \\\hline
0.5& 0.5& 0.5 &UIMC       & 60.36  & 44.45    & 68.82   \\\hline
0.25& 0.25& 1 &UIMC       & 55.03  & 43.97    & 68.05  \\\hline
0.2& 0.5& 0.8 &UIMC       & 52.07  & 50.97    & 70.41  \\\hline
0.2& 0.3& 1 &UIMC    & 57.74  & 47.68    & 69.82  \\\hline
0.1& 0.5& 0.9 &UIMC     & 60.95  & 45.41    & 68.63  \\\hline
0& 0.5& 1 &UIMC    & \textbf{62.72}  & \textbf{56.51}    & \textbf{72.78}  \\
\hline
\end{tabular}
}
\caption{Results on the 3-Sources Dataset with Average Missing Rate (PER) of 0.5 (Bold Numbers Mean the Best Results.)}
\label{3sourcesbiao}
\end{table}
\subsection{Experimental Results and Analysis}
Fig.~\ref{fig:UIMC_b} shows the balanced incomplete multi-view clustering results and Fig.~\ref{fig:UIMC_unb} gives the unbalanced incomplete multi-view clustering results. Obviously, UIMC performs best on all datasets, which shows that UIMC is not sensitive to the dataset.
Impressively, compared with the best performing method among compared methods on Digit dataset with PER=0.4 (Fig.~\ref{fig:UIMC_unb}(g)-\ref{fig:UIMC_unb}(h)), UIMC at least raises ACC by 44.50\%, NMI by 42.70\%, purity by 43.15\%.

\noindent\textbf{1) Balanced incomplete multi-view clustering}: In Fig.~\ref{fig:UIMC_b}, our proposed UIMC outperforms significantly compared methods for all missing rates. Specifically, relative to the compared methods, UIMC at least improves clustering performance (ACC+NMI+Purity)  by 17.63\% with PER=0.1, by 15.42\% with PER=0.3, by 11.06\% with PER=0.5.
It is mainly because based on $L_{2,1}$-norm, UIMC can learn the robust representation to decrease the influence of noises in BUAA.
Also, the superior improvement proves that UIMC performs better than other state-of-the-art methods on balanced incomplete multi-view clustering tasks.
Compared with PVC, IMG raises clustering results when PER=0.1 and PER=0.3. The reason is that based on a graph Laplacian term, IMG can learn the global structure. Obviously, UIMC obtains better clustering performance because UIMC aligns two incomplete views by minimizing the disagreement between their cluster indicator matrix and the consensus.
This shows that UIMC can handle balanced incomplete multi-view clustering tasks more effectively than the other state-of-the-art methods.

%

\noindent\textbf{2) Unbalanced incomplete multi-view clustering}: In Fig.~\ref{fig:UIMC_unb}, our proposed UIMC also obtains better performance than other state-of-the-art methods on three datasets for unbalanced incomplete clustering. As the PER increases, the clustering results of other methods drop sharply, but UIMC still keeps satisfactory clustering performance due to our proposed weighted MVSC. The main reason is that UIMC can cover these data with the help of $\Gamma$-norm, which decreases the influence of incompleteness. Interestingly, comparing the performance improvement of UIMC on 3-Sources, BBC and Digit datasets, we find that UIMC has the largest improvement on BBC and the least improvement on Digit.
It is because UIMC weights each view based on its incompleteness, which makes full use of these presented instances for clustering.
When PER$>$0, both MIC and DAIMC perform worse than BSV and Concat. It is because MIC and DAIMC rely on view alignment but they cannot align different views after concatenating all views.
This also shows that it is invalid to solve the unbalanced incomplete multi-view clustering problem by concatenating all the views.
Compared with Concat, IMSC$\_$AGL and UEAF have a small improvement. For example, when clustering the BBC dataset with PER=0.2, IMSC$\_$AGL and UEAF improve clustering results by up to 10\%, while our proposed UIMC improves the performance by at least 30\%. The large improvements show that UIMC effectively solves the unbalanced incomplete multi-view clustering problem.
The reason is that UIMC effectively bridges these incomplete views by constructing the Laplacian graphs with the same matrix size.
Moreover, it also shows our adaptive weighting strategy can adaptively assign larger weights to more available views for better clustering results.
\begin{table}[t]
\centering
\scalebox{1.0}{
\begin{tabular}{|c|cc|cc|cc|ccc}
\hline
\multirow{2}*{PER}&\multicolumn{2}{c|}{ACC (\%)}&\multicolumn{2}{c|}{NMI (\%)}&\multicolumn{2}{c|}{Purity (\%)} \\\cline{2-7}
~&full& fw&full& fw&full& fw \\\hline
0& \textbf{70.41}&65.09&\textbf{67.51}&63.66& \textbf{80.47}& 79.29\\\hline
0.1&\textbf{66.86}& 62.72& \textbf{65.31}& 57.73&\textbf{79.88}&72.78  \\\hline
0.2& \textbf{66.13}& 60.95& \textbf{64.37}&55.76&\textbf{79.29}&68.05   \\\hline
0.3&\textbf{65.68}&60.36&\textbf{62.25}&51.72&\textbf{78.70}&68.05 \\\hline
0.4&\textbf{68.05}&57.40&\textbf{61.65}&51.19&\textbf{78.11}&68.64   \\\hline
0.5&\textbf{66.87}&48.52&\textbf{51.66}&46.30&\textbf{71.60}&65.09 \\\hline
\end{tabular}
}
\caption{Performance comparison on the 3-Sources Dataset. ``full''  represents our full UIMC. ``fw'' means that we remove Eq.~\eqref{gengw} and only obtain the view weight through Eq.~\eqref{w_{v}} to perform UIMC. (Bold Numbers Mean the Best Results.)}
\label{jizhunshiyan}
\end{table}
\subsection{Weight Study}
Some incomplete multi-view clustering methods also weigh each view to identify its importance. In this section, we compare our proposed UIMC with these methods in detail. MIC and DAIMC use a diagonal incomplete indicator clustering matrix as a weight matrix. But they cannot directly quantify the importance of different views from this matrix. Borrowing the idea of BMVC \cite{8387526}, UEAF weighs different views for clustering. Since the weighting strategy of UEAF is the most similar to ours, we compare our proposed UIMC with UEAF.

Firstly, we will study the impact of unbalanced incompleteness on the view availability. Then, we will compare the weight learned by UEAF and the weight learned by UIMC on unbalanced incomplete multi-view clustering.
\subsubsection{Impact of Unbalanced Incompleteness on the View Availability}
To visually demonstrate the impact of unbalanced incompleteness on each view, we perform UIMC on each view of both balanced incomplete multi-view datasets (shown in Fig.~\ref{fig:dangeshitu_ba}) and unbalanced incomplete multi-view datasets (shown in Fig.~\ref{fig:dangeshitu_unba}). Note that the balanced incomplete multi-view clustering performance (i.e., Fig.~\ref{fig:dangeshitu_ba}) is the baseline. For each balanced incomplete multi-view dataset, we randomly delete the same number of instances from each view to obtain balanced incomplete views.

From Fig.~\ref{fig:dangeshitu_ba}, we can find that the relative availability of each view remains unchanged in most cases. For example, when clustering the BBC dataset, Bb has the highest availability on both PER=0.1 and PER=0.2. Moreover, under the same PER, the availability (or performance) of different views does not differ by more than 20\%. For instance, when clustering the 3-Sources dataset on PER=0.1, the difference between the three views is just 10\%.

From Fig.~\ref{fig:dangeshitu_unba}, we can notice that the view with a lower missing rate (i.e., strong view) tends to have higher availability than the view with a higher missing rate (i.e., weak view). For example, when clustering the BBC dataset on PER=0.5, Ba has the highest availability among the four views. It is because the incompleteness of Ba is the lowest (shown in Table~\ref{datashezhi}), and Ba has the most available information for clustering. Besides, as PER increases, the difference between the strong view and the weak view will increase. For instance, when clustering the 3-Sources dataset, in terms of NMI, BBC outperforms GUA by about 5\% on PER=0 and by about 30\% on PER=0.5. Moreover, the difference also illustrates the validity of Definition \ref{def_qiangruo} (i.e., ``Strong view'' and ``Weak view''). To further show the validity of our definitions, we report the view completeness in Fig.~\ref{fig:unbanlanced_zhu_wanzhengxing} and \ref{fig:unbanlanced_zhu_bbc_wanzhengxing}. Obviously, when PER$>0$, the view with higher completeness (lower incompleteness) obtains better performance. For example, when PER=0.4 on the 3-Sources dataset, BBC has the highest completeness and the best clustering performance, while GUA has the lowest completeness and the worst clustering performance. Thus, we can think that BBC is a strong view and GUA is a weak view, which is consistent with our definition in Definition \ref{def_qiangruo}.

One can ask what if changing the missing rate vector in Table~\ref{datashezhi}. Thus, we perform UIMC on 3-Sources Dataset with PER=0.5. The clustering results are shown in Table~\ref{3sourcesbiao}, where BSV and Concat are the baseline methods. Obviously, UIMC outperforms baseline methods in all the cases. Moreover, when the missing rates of BBC, REU, and GUA are 0, 0.5, and 1, respectively, UIMC obtains the best performance. The main reason is that based on the adaptive weighting strategy, UIMC can make full use of the information of each view for clustering. UIMC's satisfactory performance also shows the effectiveness of our proposed view evolution. Note that GUA is a dying view, which is removed by Eq.~\eqref{w_{v}}. Thus, we only integrate two preserved views (i.e., BBC and REU) for clustering. This also shows that we can directly handle several unbalanced incomplete multi-view clustering cases.

To further demonstrate the effectiveness of our weighting strategy (i.e., Eq.~\eqref{w_{v}} and Eq.~\eqref{gengw}), we set a baseline: fixing each weight (called ``fw'').
For fw, we remove Eq.~\eqref{gengw} and only obtain the view weight through Eq.~\eqref{w_{v}}. We call our full UIMC ``full''. We compare ``full'' with ``fw'' on the 3-Sources dataset.
Table~\ref{jizhunshiyan} shows the performance comparison.
Obviously, ``full'' outperforms ``fw'' in all the cases, which verifies the effectiveness of our weighting strategy. 

\subsubsection{Weight Comparison between UEAF and UIMC}
To further demonstrate the rationality of the weights we learn, we compare our proposed UIMC with UEAF. In the comparison, we perform UIMC and UEAF on three unbalanced incomplete multi-view datasets: 3-Sources, BBC and Digit. But UEAF cannot directly cluster these unbalanced incomplete data. Similar to BSV and Concat, we fill in the missing views with the average of instances in the corresponding views to obtain complete views, then perform UEAF on the complete multi-view dataset. For UIMC, we directly perform it on unbalanced incomplete multi-view datasets. Based on this experiment, we will compare the weights learned by UEAF and the weights learned by UIMC. The results are shown in Table~\ref{quanzhongduibi}.

Combining Fig.~\ref{fig:dangeshitu_unba} and Table~\ref{quanzhongduibi}, we can find that though UEAF can learn a weight for each view, this weight does not reflect the availability of a view. For example, when using UEAF to cluster the BBC dataset on PER=0.5, the weight of Ba is 0.0198 and the weight of Bd is 0.9120. In fact, as shown in Fig.~\ref{fig:unbanlanced_zhu_bbc_acc}-Fig.~\ref{fig:unbanlanced_zhu_bbc_pur}, Ba has higher availability than Bd. The reason is that UEAF ignores the impact of unbalanced incompleteness on view availability.

Impressively, our proposed UIMC can adaptively assign proper weight to each view, which can accurately reflect the availability of the view. For instance, when using UIMC to cluster the BBC dataset on PER=0.5, the weight of Ba is 0.4371 and the weight of Bd is 0.0510. As shown in Fig.~\ref{fig:unbanlanced_zhu_bbc_acc}-Fig.~\ref{fig:unbanlanced_zhu_bbc_pur}, Ba has higher availability than Bd, which also shows the effectiveness of our weight strategy.

In summary, when clustering unbalanced incomplete multi-view datasets, the weight learned by UIMC is more accurate than the weight learned by UEAF.

\subsection{Parameter Sensitivity and Convergence Study}
\begin{figure*}[t]
\centering
\subfigure[Study for $\alpha$ and $\beta$]{\label{fig:nmi_3d} \includegraphics[width=0.3\textwidth]{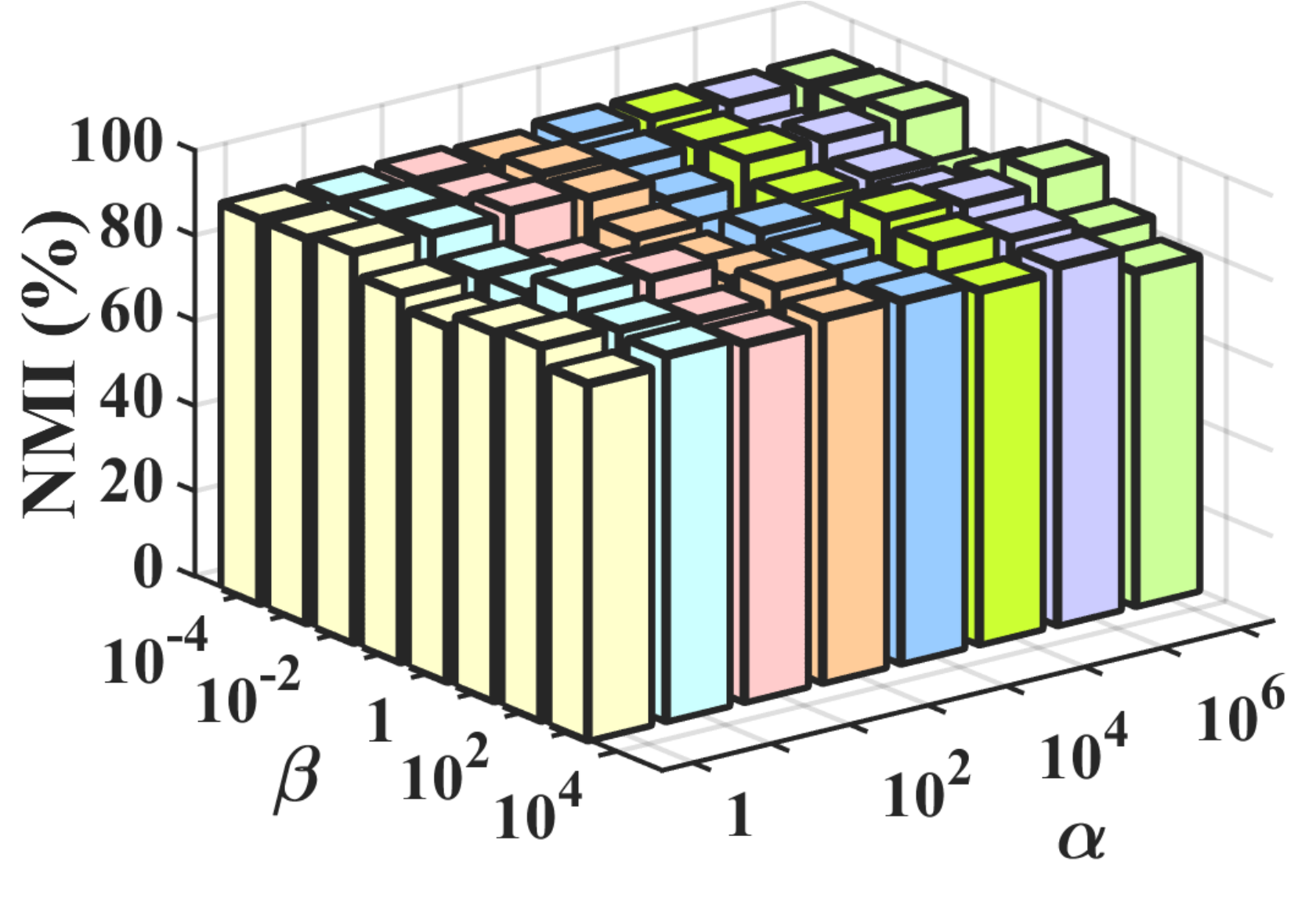} }
\subfigure[Study for $\eta$]{\label{fig:nmi_2d} \includegraphics[width=0.3\textwidth]{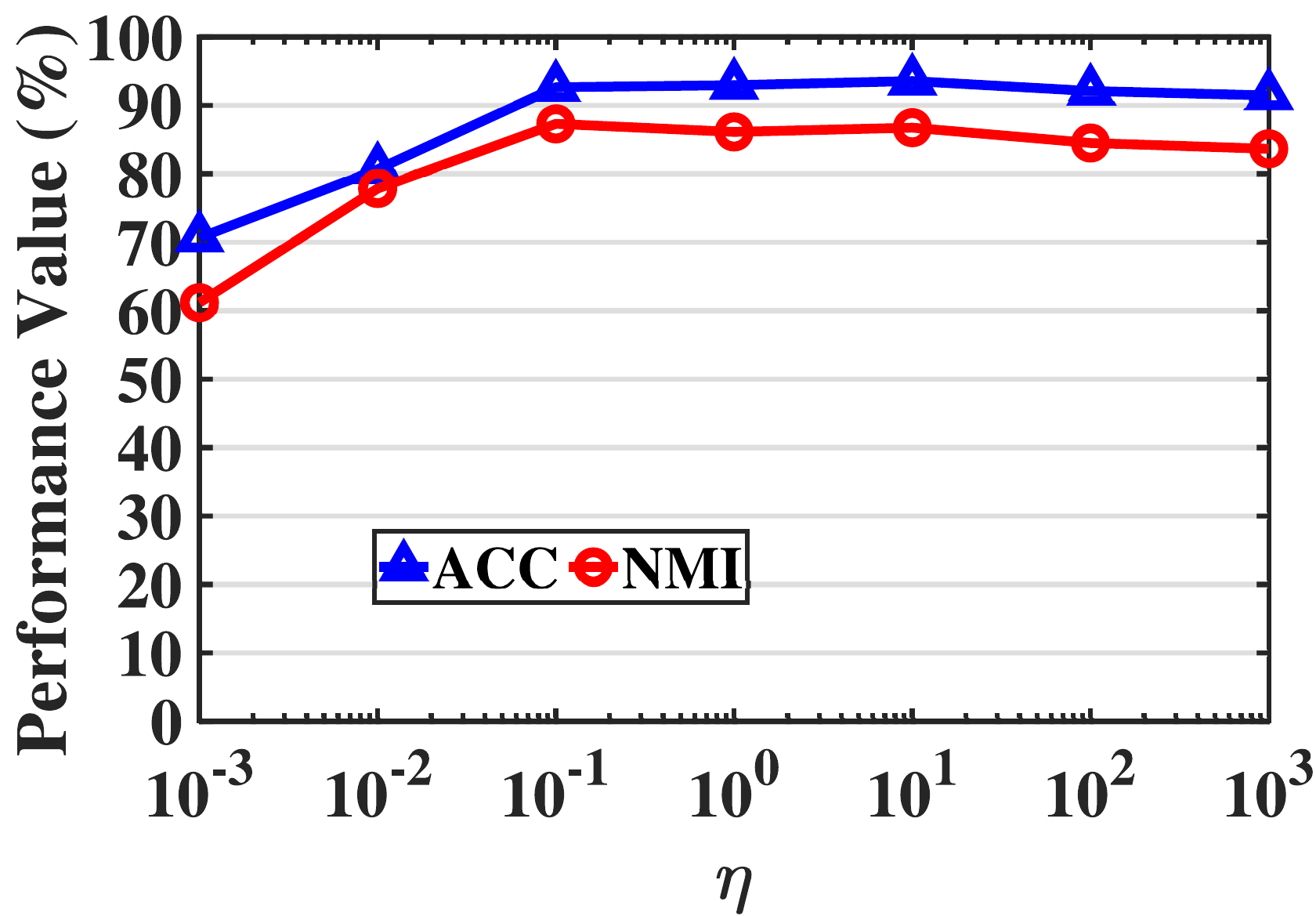}}
\subfigure[Convergence study]{\label{fig:plot_obj} \includegraphics[width=0.3\textwidth]{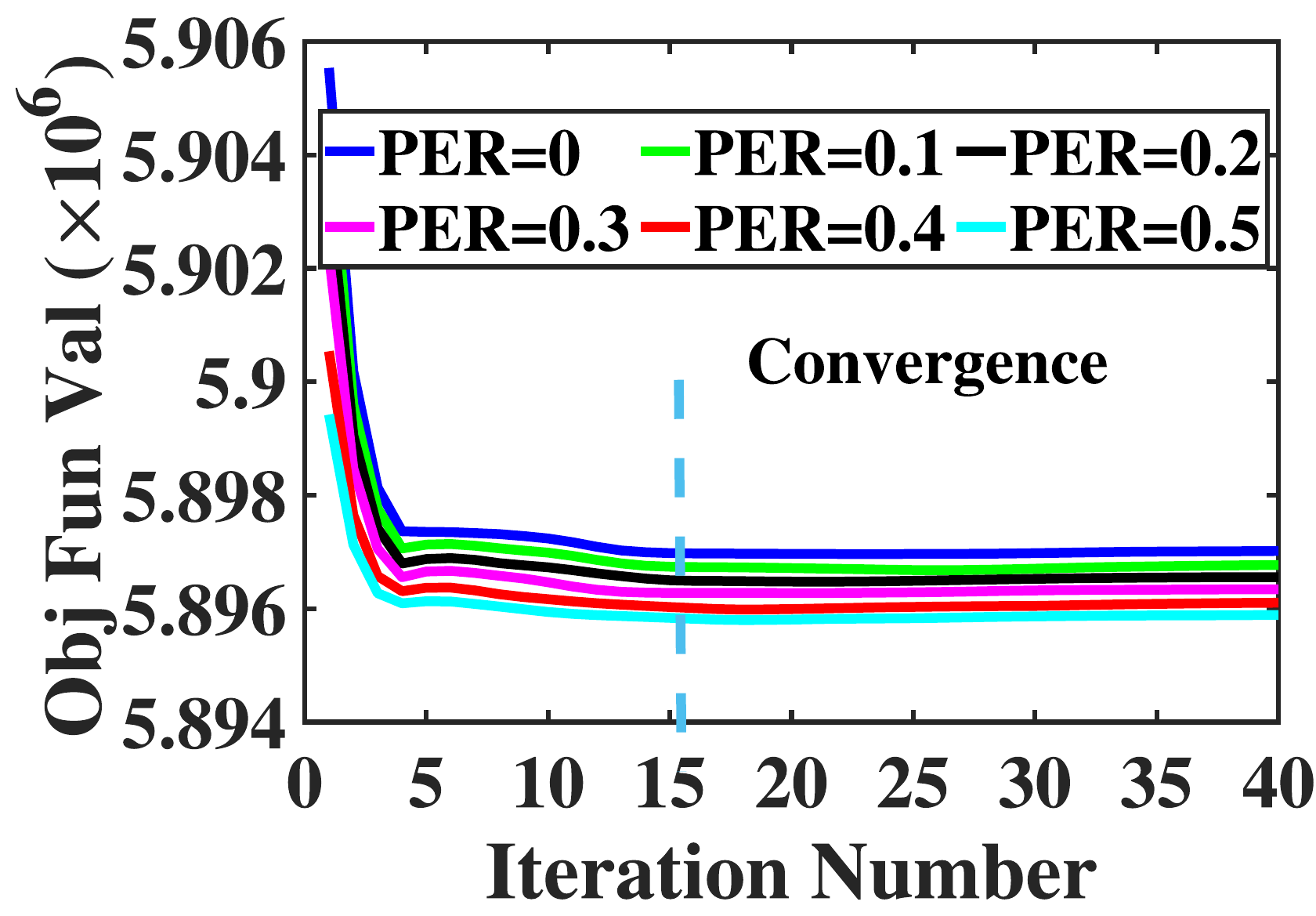} }
\caption{Parameter sensitivity and convergence study, where ``Obj Fun Val'' donates ``Objective function value''.}
\label{fig:UIMC_minganxing}
\end{figure*}
\begin{table}[t]
\centering
\scalebox{0.9}{
\setlength{\tabcolsep}{1.5mm}{
\begin{tabular}{c|ccccccccc}
\hline
\multicolumn{7}{c}{The BBC dataset}  \\\hline
PER& 0& 0.1&0.2 & 0.3& 0.4& 0.5  \\\hline
Time ($s$)&162.2325 & 153.9034& 146.9688& 141.0714 &140.3818 & 135.5273 \\\hline
\multicolumn{7}{c}{The Digit dataset}  \\\hline
PER& 0& 0.1&0.2 & 0.3& 0.4& 0.5  \\\hline
Time ($s$)& 730.9850& 561.5042& 435.9347& 349.0930& 303.1353& 261.1742\\\hline
\end{tabular}
}
}
\caption{Convergence time (Time ($s$)) under each average missing rate (PER) on BBC and Digit.}
\label{runtime}
\end{table}
We study the sensitivity of parameters ($\alpha$, $\beta$ and $\eta$) by conducting unbalanced incomplete multi-view clustering on the Digit dataset. We set PER=0.3 and report the clustering performance of UIMC versus $\alpha$, $\beta$ and $\eta$. From Fig.~\ref{fig:nmi_3d} and~\ref{fig:nmi_2d}, as these parameters change, UIMC keeps stable and satisfactory clustering performance. It illustrates that UIMC is insensitive to these parameters to some extent.
Besides, UIMC obtains the best clustering results when $\alpha$=$10^{-2}$, $\beta$=$10^{5}$ and $\eta$=$10^{-1}$, which are our recommended values.

We study UIMC's convergence by experimenting on the BBC dataset with different PERs. We set the hyper-parameters ${\alpha,\beta,\eta}$ as $10^{-2}, 10^{5}, 10^{-1}$, respectively. Fig.~\ref{fig:plot_obj} shows that the convergence curve versus the number of iterations. The blue dashed line indicates the convergence of the objective function.
It can be seen that UIMC has converged just after $16$ iterations for all PERs, which verifies the convergence of UIMC. Table~\ref{runtime} also shows the convergence time in different cases. Note that as PER increases, the convergence time will decrease. It is because when the missing rate increases, the data matrix will become more sparse. This sparse matrix can speed up calculations \cite{chen2018performance}.


\section{Conclusion}  \label{section:con}
In this paper, we propose a novel unbalanced incomplete multi-view clustering method, named UIMC. To our best knowledge, it presents the first effective method to cluster multiple views with different incompleteness. Inspired by the biological evolution theory, we propose the scheme of view evolution to integrate these unbalanced incomplete views for clustering.
After each iteration of optimization,
the weights of strong views increase, while the weights of the weak views decrease.
Extensive experiments on four real-world multi-view datasets show the superior performance gain and effectiveness of UIMC.
Impressively, compared with the baseline method Concat on the Digit dataset with the missing rate of 0.4, other state-of-the-art methods at least improve ACC by 5.20\%, NMI by 7.33\% and purity by 8.30\%, while our proposed UIMC at least improves ACC by 51.35\%, NMI by 51.26\% and purity by 51.20\%. In the feature, we will introduce online learning to reduce the  computational cost of UIMC.
%

\ifCLASSOPTIONcaptionsoff
  \newpage
\fi
\bibliographystyle{IEEEtran}
\bibliography{IEEEtran}
\end{document}